%% file: main.tex
\documentclass[12pt]{graphicsclass}
\usepackage[utf8]{inputenc}
\usepackage[T1]{fontenc}
\usepackage{palatino}
\usepackage{graphicx}
\usepackage{setspace}
\usepackage[margin=1in]{geometry}
\usepackage[hidelinks]{hyperref}
\usepackage{cite}
\usepackage{xcolor}
\usepackage{overpic}
\usepackage{macros}
\usepackage[all]{nowidow}

\title{
{Contrastive Learning as Kernel Approximation}\\~\\~\\
{\large Konstantinos Christopher Tsiolis \\
    Department of Mathematics and Statistics \\ 
	McGill University, Montreal \\ 
	August 2023 \\~\\~\\
	A thesis submitted to McGill University in partial fulfillment of the requirements of the degree of \\~\\ Master of Science }\\~\\
}

\author{\textcopyright Konstantinos Christopher Tsiolis, 2023}
\date{}

\begin{document}
\maketitle

\chapter*{Abstract}
\label{sec:engAbstract}
\addcontentsline{toc}{section}{\nameref{sec:engAbstract}}

\input{Chapters/abstract_en.tex}

\chapter*{Abrégé}
\label{sec:frAbstract}
\addcontentsline{toc}{section}{\nameref{sec:frAbstract}}

\input{Chapters/abstract_fr.tex}

\chapter*{Acknowledgements}
\label{sec:ded}
\addcontentsline{toc}{section}{\nameref{sec:ded}}

\input{Chapters/acknowledgements.tex}

\tableofcontents
\listoffigures %
\addcontentsline{toc}{section}{\listfigurename}

\clearpage 
\pagenumbering{arabic} 

\input{Chapters/intro.tex}

\input{Chapters/ml_prelims.tex}

\input{Chapters/rkhs.tex}

\input{Chapters/dim_reduction.tex}

\input{Chapters/contrastive.tex}

\input{Chapters/conclusion.tex}

\bibliography{references}
\bibliographystyle{acm}

\end{document}

%% file: Chapters/abstract_en.tex
In standard supervised machine learning, it is necessary to provide a label for every input in the data. While raw data in many application domains is easily obtainable on the Internet, manual labelling of this data is prohibitively expensive. To circumvent this issue, contrastive learning methods produce low-dimensional vector representations (also called features) of high-dimensional inputs on large unlabelled datasets. This is done by training with a contrastive loss function, which enforces that similar inputs have high inner product and dissimilar inputs have low inner product in the feature space. Rather than annotating each input individually, it suffices to define a means of sampling pairs of similar and dissimilar inputs. Contrastive features can then be fed as inputs to supervised learning systems on much smaller labelled datasets to obtain high accuracy on end tasks of interest. 

The goal of this thesis is to provide an overview of the current theoretical understanding of contrastive learning, specifically as it pertains to the minimizers of contrastive loss functions and their relationship to prior methods for learning features from unlabelled data. We highlight popular contrastive loss functions whose minimizers implicitly approximate a positive semidefinite (PSD) kernel. The latter is a well-studied object in functional analysis and learning theory that formalizes a notion of similarity between elements of a space. PSD kernels provide an implicit definition of features through the theory of reproducing kernel Hilbert spaces.

%% file: Chapters/abstract_fr.tex
L’apprentissage automatique supervisé nécessite que tous les données soient étiquetées. Dans plusieurs domaines, des données sont facilement obtenues sur Internet, mais l’étiquetage manuel est long. Les méthodes contrastives surmontent cet obstacle. Elles apprennent des vecteurs de caractéristiques de petite dimension pour des exemples non étiquetés situés dans un espace de grande dimension. L’entraînement s’effectue avec une fonction de perte contrastive qui encourage que les exemples similaires aient un produit scalaire élevé et que les exemples non similaires aient un produit scalaire faible dans l’espace des caractéristiques. Au lieu d’étiqueter chaque exemple manuellement, il suffit de spécifier une façon d'échantillonner des paires d'exemples similaires ainsi que des paires d'exemples non similaires. Les vecteurs de caractéristiques qui résultent de l’apprentissage contrastif peuvent ensuite être utilisés comme entrées pour des systèmes supervisés entraînés sur un petit nombre d’exemples afin d’obtenir un faible taux d’erreur pour divers tâches d’apprentissage.

L’objectif de cette thèse est de faire le survol des connaissances théoriques concernant l’apprentissage contrastif. Nous concentrons surtout sur les minimiseurs des fonctions de perte contrastives et la relation entre l'apprentissage contrastif et des méthodes moins récentes qui apprennent des vecteurs de caractéristiques à partir de données non étiquetées. Nous soulignons des exemples populaires de fonctions de perte contrastives dont les minimiseurs effectuent implicitement l’approximation d’un noyau symmétrique et defini positif. Ce dernier est traité dans l’analyse fonctionelle et la théorie de l’apprentissage automatique. Il formalise la notion de similarité entre les éléments d’un espace et fournit une définition implicite des caractéristiques grâce à la théorie des espaces de Hilbert à noyau reproduisant.

%% file: Chapters/acknowledgements.tex
While this thesis is the product of the work that I have conducted during my master's degree at McGill, it is also a tribute to all those who have helped me get to this point. Without them, this would not be possible.

First and foremost, I would like to thank my parents for providing me with everything I need to pursue my dreams. To Mom, thank you for the countless hours you spent helping me with my homework in elementary school, and thank you for teaching me to always do more than what is expected of me. To Dad, thank you for driving me to my high school through rush hour traffic every morning, and thank you for being my coach not only in soccer but also in life. 

I am deeply grateful to my supervisor, Prof.\@ Adam Oberman, who shares my enthusiasm for contrastive learning and has mentored me since my time as an undergraduate student. In our lab group, I was fortunate to receive guidance on implementing machine learning pipelines and designing experiments from Dr.\@ Tiago Salvador and Vikram Voleti. I would also like to thank Noah Marshall for our many discussions on the topics addressed in this thesis and above all for being an amazing friend.

I further wish to acknowledge the research environment at McGill and Mila that has encouraged me to expand my knowledge and work on multiple projects over the course of my degree. Thank you to Profs.\@ Courtney and Elliot Paquette for establishing a student seminar on random matrix theory, machine learning, and optimization. This was a tremendous help as I learned how to parse theoretical research and present it. I am also excited to be working on a project on random matrix theory for neural networks with Hugo Latourelle-Vigeant and Prof.\@ Elliot Paquette\footnote{This project is equally interesting in my opinion, but will not be addressed in this thesis.}. 

I would also like to thank all of the mentors and collaborators I worked with before my master's studies. I am grateful to Prof.\@ Jackie Cheung, who supervised me while I was an undergraduate student for two summers. During those summers, I was fortunate to work with Dr.\@ Edward Newell, Jingyi (Kylie) He, and Kian Kenyon-Dean on word embeddings. This was my first formal research experience, and it sparked my interest in contrastive learning. Also important to my development was the support I received from Dr.\@ Ivan Ivanov and Sandi Mak, two of my mathematics teachers at Vanier College, who introduced me to machine learning.

And of course, where would I be without my family and friends? Thank you to Melina and Stephania for being the best sisters I could ever ask for. Thank you to my grandparents for the sacrifices that they made in emigrating to Canada to provide their children and grandchildren with opportunities like the one I have now. Thank you to my friends for always being there for me.

Finally, I would like to acknowledge the funding that I received from the Natural Sciences and Engineering Research Council of Canada (NSERC) through the CGS M scholarship and from the Fonds de recherche du Québec - Nature et technologies (FRQNT) through the B1X Scholarship.

%% file: Chapters/intro.tex
\chapter{Introduction}\label{chapter:intro}

Substantial progress has been made in the field of machine learning in the past decade, and deep learning\footnote{Here, \textit{deep learning} is understood to mean learning with artificial neural networks.} \cite{lecun2015deep,goodfellow2016deep} stands out as the major contributing factor. Through the composition of vector-valued functions, neural network models map high-dimensional inputs (e.g., images, words) to low-dimensional representations --- also called \textit{features} or \textit{embeddings} --- in Euclidean space. Though there is no exact definition of what makes a "good" representation \cite{bengio2013representation}, a plethora of heuristics (e.g., loss functions, neural network architectures, regularizers, sampling strategies) have been developed to train neural network feature encoders.

The effectiveness of representation learning is a consequence of the \textit{generalization} ability of neural networks. The latter are parametric function approximators that produce an output for each input in the source domain, regardless of whether or not the input appears in the training data. Across a variety of tasks and application domains, predictors which make use of neural network features achieve superior performance\footnote{The way performance is measured varies according to the task and application domain (e.g., accuracy for classification tasks, perplexity for language modelling, cumulative reward for control tasks).} on unseen data than those which do not \cite{bengio2000neural,mnih2015human,krizhevsky2017imagenet}. 


The use of convolutional neural networks (CNNs) \cite{lecun1998gradient,krizhevsky2017imagenet,he2016deep} and Vision Transformers (ViTs) \cite{dosovitskiy2021image} to learn features on large datasets such as ImageNet \cite{russakovsky2015imagenet} for image classification is now common practice. Such features are also employed to generate new images by training \textit{decoder} neural networks to map from the feature space back to the input space \cite{kingma2014auto,goodfellow2020generative,ho2020denoising,ramesh2022hierarchical}.

In natural language processing (NLP), representation learning has replaced feature engineering, the manual specification of features which are believed to be relevant to the task at hand (e.g.\@ word frequencies, parts of speech, parse tree attributes) \cite{bikel1999algorithm,florian2003named,higgins2003machine}. Word embedding methods \cite{mikolov2013distributed,pennington2014glove} automatically produce vector representations of words based on statistics from a large collection of text (called a \textit{corpus}). More recently, large neural network language models learn context-informed representations of words that are applied to a variety of linguistic tasks, such as question answering, translation, and text generation \cite{peters2018deep,devlin2019bert,brown2020language}.

In reinforcement learning, neural networks predict optimal actions to solve control tasks without an explicit model of the environment \cite{sutton2018reinforcement}. Notable successes include achieving human-level performance on Atari games \cite{mnih2015human} and defeating Lee Sedol, one of the world's top Go players \cite{silver2016mastering}.

Deep learning has applications to medicine \cite{esteva2021deep,ahsan2022machine}, biology \cite{jumper2021highly}, the mitigation of climate change \cite{rolnick2022tackling}, finance \cite{dixon2020machine}, and education \cite{kuvcak2018machine}, among many others. However, its increasing applicability raises important concerns. Deep learning systems can amplify biases based on gender and race, among others \cite{mehrabi2021survey}. \cite{buolamwini2018gender} found that several commercial facial recognition systems are substantially less accurate in identifying darker-skinned females compared to lighter-skinned males, a discrepancy that they attributed to an unbalanced training dataset. A journalistic investigation deemed COMPAS, a software consulted by judges in American courts to predict recidivism risk, to be biased against African Americans \cite{angwin2016machine}. Word embeddings trained with the word2vec algorithm \cite{mikolov2013distributed} (which we will discuss in Section \ref{section:word_embeddings}), exhibit stereotypical associations between gender and occupation (e.g. "woman" and "homemaker") \cite{bolukbasi2016man}. 

The widespread adoption of deep learning has heightened the need for a more complete understanding of these methods in order assess risks and failure modes such as the ones described above. Deep learning is now the backbone of algorithms that govern hiring, justice, and healthcare decisions. Neural networks are often viewed as black-box function approximators, and it is thus difficult to explain the factors leading to a decision for a specific input \cite{tjoa2020survey,li2023trustworthy}. For example, in the case of word embeddings, there is no immediate linguistic interpretation of each of the components of a word's vector representation \cite{csenel2018semantic}. 

Statistical learning theory \cite{valiant1984theory,vapnik1999overview,shalev2014understanding,mohri2018foundations} and the more recent deep learning theory \cite{belkin2021fit} seek to explain the generalization ability of learning systems. These theories are primarily focused on \textit{supervised learning}, which solely leverages labelled data for the prediction task at hand. In this setting, neural networks are trained \textit{end-to-end}: the mapping from input space to feature space and the mapping from feature space to predictions are learned simultaneously. 

When labelled data is scarce, a common approach is to disentangle feature learning and prediction. While the latter requires labelled data, the former step can leverage the wealth of unlabelled data that is available on the Internet for many domains of interest (e.g. text, images). \textit{Unsupervised learning} techniques operate on unlabelled data to glean the structure of the input space. For example, dimensionality reduction techniques such as multidimensional scaling (MDS) \cite{cox2000multidimensional}, ISOMAP \cite{tenenbaum2000global}, Locally Linear Embeddings (LLE) \cite{roweis2000nonlinear}, and Laplacian Eigenmaps \cite{belkin2003laplacian} learn an approximate distance-preserving map between the input space and a low-dimensional feature space. The resulting features can then serve as initialization for a supervised model applied to end tasks of interest (called \textit{downstream tasks}). However, this paradigm has important limitations. It requires that a similarity (or distance) function be pre-specified for every pair of inputs, which is unfeasible for large datasets. Furthermore, many approaches incorporate a computationally expensive spectral decomposition step and are unable to generalize to new inputs.

\textit{Self-supervised learning} (SSL) mitigates the issues outlined above. It is similar to unsupervised learning in that it does not assume access to labels, but it distinguishes itself by relying on additional knowledge that can be automatically extracted or generated. This knowledge procures labels for an auxiliary supervised task which is believed to be relevant to the downstream tasks. Solving the auxiliary task produces features for downstream models. 

We focus on \textit{contrastive learning}, an umbrella term for SSL methods which define the auxiliary task of classifying pairs of inputs as similar or dissimilar. \textit{Contrastive loss functions} capture the intuitive idea that representations for similar inputs should be close together while representations for dissimilar inputs should be far apart \cite{chopra2005learning,hadsell2006dimensionality,gutmann2012noise,mikolov2013distributed,sohn2016improved,oord2018representation,chen2020simple,haochen2021provable}. Unlike unsupervised methods, it is not necessary to explicitly define a similarity function on all pairs of inputs. Instead, one specifies a procedure for sampling similar (\textit{positive}) and dissimilar (\textit{negative}) pairs. The definition of positive and negative pairs varies according to the application area of interest. For word embedding algorithms such as word2vec \cite{mikolov2013distributed}, two words which appear together in a text corpus form a positive pair. In computer vision, a common choice is to designate two transformed versions (called \textit{views}) of the same source image as a positive pair \cite{oord2018representation, wu2018unsupervised, hjelm2019learning, bachman2019learning, misra2020self, henaff2020data, chen2020simple}. We highlight SimCLR \cite{chen2020simple}, a contrastive learning algorithm for computer vision upon which many theoretical analyses are based \cite{wang2020understanding,zimmermann2021contrastive,wen2021toward,haochen2021provable,balestriero2022contrastive,johnson2022contrastive}. On a suite of downstream tasks, \cite{chen2020simple} determined that SimCLR achieves higher accuracy than fully supervised training.

There is a growing literature which aims to explain the empirical success of contrastive learning \cite{arora2019theoretical,tosh2021contrastive_a,tosh2021contrastive_b,wang2020understanding,wen2021toward,zimmermann2021contrastive,haochen2021provable,balestriero2022contrastive,johnson2022contrastive}. This thesis is intended to summarize these theoretical developments, with emphasis on the minima of contrastive loss functions and the relationship between contrastive learning and dimensionality reduction. We provide the necessary background to interpret recent results and motivate future work. 

Contrastive learning was first developed as a remedy to computationally expensive unsupervised dimensionality reduction methods, many of which lack the ability to generalize \cite{chopra2005learning,hadsell2006dimensionality}. Recent theoretical work established a connection between the two \cite{haochen2021provable,balestriero2022contrastive,johnson2022contrastive}. As a result, contrastive learning can be interpreted as a parametric approach for approximating a notion of similarity between elements the input space. This notion of similarity is captured by a positive semidefinite (PSD) kernel, a well-understood object in functional analysis. The theory of reproducing kernel Hilbert spaces (RKHS) establishes that PSD kernels implicitly define a mapping from the input space to a Hilbert space of functions, which we can in turn interpret as features. This insight has led to the development of algorithms that explicitly learn features from PSD kernels with neural networks \cite{deng2022neuralef,deng2022neural}. 

The thesis is structured as follows. Chapter \ref{chapter:ml_prelims} covers machine learning fundamentals. We review the paradigms of supervised, unsupervised, and self-supervised learning before presenting a formal definition of neural networks and loss functions. With this in hand, we outline the \textit{logistic regression} method for classification, which features prominently in solving the auxiliary task defined by contrastive learning. Chapter \ref{chapter:rkhs} concerns RKHS theory, with emphasis on the results that connect PSD kernels to feature maps. In Chapter \ref{chapter:dim_reduction}, we discuss dimensionality reduction methods, namely principal component analysis (PCA) \cite{pearson1901on,hotelling1933analysis}, MDS, ISOMAP, LLE, Laplacian Eigenmaps, the Nyström method \cite{drineas2005nystrom}, and Random Fourier Features \cite{rahimi2007random}. Chapter \ref{chapter:contrastive} introduces contrastive methods and explains how they overcome the limitations of unsupervised dimensionality reduction.
We provide examples of contrastive methods and loss functions before addressing the theoretical results linking PSD kernels, dimensionality reduction, and contrastive learning.

%% file: Chapters/ml_prelims.tex
\chapter{Machine Learning Preliminaries}\label{chapter:ml_prelims}

In this chapter, we provide the machine learning background that is necessary for our subsequent discussion of contrastive learning. In Section \ref{section:paradigms}, we provide a formalism to capture the learning problems that we consider. We describe the supervised, unsupervised, and self-supervised learning paradigms. In Section \ref{section:classification}, we review the theory of binary and multiclass classification. The results we present therein will be useful in Chapter \ref{chapter:contrastive} to address the link between contrastive learning and kernel approximation.

\section{Paradigms and Terminology}\label{section:paradigms}

Machine learning is concerned with leveraging data to make accurate predictions on tasks of interest. The MNIST problem \cite{lecun1998gradient} is a famous example. The objective is to find a classifier which can determine the identity of a handwritten digit with the lowest possible error. The machine learning approach advocates using available data (which forms a \textit{dataset}) to define a suitable classification rule. If the dataset contains both images of handwritten digits and their corresponding identities (called \textit{labels} or \textit{targets}), we are in the \textit{supervised} setting. If the dataset instead contains only the images, we are in the \textit{unsupervised} setting.  

\textbf{Supervised Learning.} We establish a formalism for supervised learning based on \cite{mohri2018foundations,shalev2014understanding}. Let $\Xcal$ denote the input/data space (e.g., the space of RGB images, the space of words) and $\Ycal$ denote the label space. We focus our attention on classification problems, where $\Ycal$ is finite (e.g., it is the set $\{0,\dots,9\}$ for MNIST) with cardinality corresponding to the number of classes $C$. We refer to the case $C = 2$ as \textit{binary} classification and the case $C > 2$ as \textit{multiclass} classification. For convenience, for every $m \in \N$, we use the notation $[m] := \{1,\dots,m\}$. Now, consider the data generating process defined below.

\begin{definition}[Labelled Dataset]
    Let $p$ be a probability distribution over $\Xcal \times \Ycal$. A collection $\Dcal := \{(x_i,y_i)\}_{i=1}^N$ of samples drawn i.i.d.\@ from $p$ is called a \textbf{labelled dataset} drawn from $p$.
\end{definition}

\begin{remark}\label{rmk:learning_theory}
    In our setting, it is convenient to consider supervised problems as approximating $p$ (or, more accurately, the conditional distributions $p(\cdot|x)$ for all $x \in \Xcal$). However, in the theoretical frameworks of \cite{mohri2018foundations,shalev2014understanding}, it is assumed that there exists a deterministic ground truth labelling function $h^*: \Xcal \to \Ycal$ that we wish to approximate. The dataset is then generated by sampling inputs $x_1,\dots,x_N$ i.i.d.\@ from a distribution over $\Xcal$ and taking the labels to be $h^*(x_1),\dots,h^*(x_N)$. If $h^*$ were allowed to be random, then this formulation would be equivalent to ours.
\end{remark}

\begin{definition}[Hypothesis Class]
    Let $\Hcal$ be a collection of functions of the form $h: \Xcal \to \pi_C$, where $\pi_C$ denotes the $C$-dimensional probability simplex. Then, $\Hcal$ is called a \textbf{hypothesis class} and each $h \in \Hcal$ is called a \textbf{hypothesis}. 
\end{definition}

\begin{definition}[Feedforward Neural Network]
    Let $L \in \Nbb$. We define a \textbf{feedforward neural network} (FFNN) of $L$ layers $\hbf^{(L)}: \Xcal \to \pi_C$ inductively as follows. Assume $\Xcal \subseteq \R^{n_0}$ for some $n_0 \in \N$. Let $n_1,\dots,n_{L-1} \in \N$ and $n_L = C$. For every $l \in [L]$, let $\Wbf_l$ be a real $n_l \times n_{l-1}$ matrix (called a \textbf{weight matrix}) and $\bbf_l \in \R^{n_l}$ (called a \textbf{bias vector}). Further let $f_l: \R \to \R$, $l \in [L-1]$, and $\tilde{\fbf}: \R^C \to \pi_C$ be nonlinear functions (called \textbf{activation functions}). Define for each $\xbf \in \Xcal$ and $l \in [L]$
    \begin{equation}
        \hbf^{(l)}(\xbf) = 
        \begin{cases}
            f_1(\Wbf_1 \xbf + \bbf_1), & l = 1 \\
            f_l\big(\Wbf_l \hbf^{(l-1)}(\xbf) + \bbf_l \big), & 1 < l < L \\
            \tilde{\fbf}\big(\Wbf_L \hbf^{(L-1)}(\xbf) + \bbf_L\big), & l=L, 
        \end{cases}
    \end{equation}
    where $f_1,\dots,f_{L-1}$ are understood as being applied elementwise to vectors.
\end{definition}
The architecture of an FFNN is specified by the number of layers $L$, the layer sizes $n_1,\dots,n_{L-1}$, and the activation functions $f_1,\dots,f_{L-1},\tilde{\fbf}$. An example of a hypothesis space is then the space of neural networks of a particular architecture.

\begin{definition}[Neural Network Encoder and Features]
    Assume $\Xcal \subseteq \R^{n_0}$. Let $\hbf^{(L)}$ be a feedforward neural network with $L \geq 2$. Then, $\hbf^{(L-1)}: \Xcal \to \R^{n_{L-1}}$ is called an \textbf{encoder}. Given $\xbf \in \Xcal$, we call $\hbf^{(L-1)}(\xbf)$ the neural network \textbf{features} (or \textbf{representation}) associated with $\xbf$. 
\end{definition}

The above definitions give us two perspectives on feedforward neural networks. First, a neural network is a composition of nonlinear functions parametrized by the weights $\Wbf_1,\dots,\Wbf_L$ and biases $\bbf_1,\dots,\bbf_L$. Second, a feedforward neural network $\hbf^{(L)}$ can be decomposed into an encoder $\hbf^{(L-1)}$ that produces a low-dimensional representation of the input and a linear classifier $\tilde{\fbf}(\Wbf_L(\cdot) + \bbf_L)$ that outputs a prediction given a representation. 

\begin{remark}
    The use of feedforward neural networks alone is uncommon in practice. Specialized neural network architectures which make additional assumptions about the nature of the input data are more popular. Convolutional neural network architectures are used extensively for image inputs \cite{lecun1998gradient,krizhevsky2017imagenet,he2016deep,chen2020simple} and Transformers are designed to handle sequential inputs \cite{vaswani2017attention,devlin2019bert,brown2020language}. However, the analyses of contrastive learning that we discuss in Chapter \ref{chapter:contrastive} abstract away the details of the neural network architecture. They are sufficiently general to handle any function which maps from $\Xcal$ to a feature space $\R^d$.
\end{remark}

Once the architecture is specified, training a neural network corresponds to minimizing a loss function over the weight matrices $\Wbf_1,\dots,\Wbf_L$ and biases $\bbf_1,\dots,\bbf_L$.

\begin{definition}[Loss Function]
    A function $\ell: \Ycal \times \pi_C \to \R^+_0$ is called a \textbf{loss function}.
\end{definition}

The first argument of the loss function is the true label $y$ for a given input $x$ and the second argument is the vector of conditional probabilities $\phat(\cdot|x)$ assigned by the model over the set of classes.

\begin{example}[Cross-Entropy Loss]
    Suppose $\Ycal = [C]$. Then, the \textbf{cross-entropy loss} is
    \begin{equation}\label{eq:xent}
        \ell(y,\yhat) = \sum_{c=1}^C \1(y=c)\log(\yhat_c),
    \end{equation}
    where $\log(\cdot)$ denotes the natural logarithm and $\yhat_c$ is the $c$-th entry of $\yhat$. This function is non-negative, convex in $\yhat$, and is minimized when $\yhat_y = 1$.
\end{example}

\begin{definition}[Expected Loss]
    Let $\ell$ be a loss function and $h \in \Hcal$. Then, the \textbf{expected loss} of the hypothesis $h$ is
    \begin{equation}
        L(h) := \E_{(x,y) \sim p} \big[\ell(y,h(x))\big].
    \end{equation}
\end{definition}

Computing an expectation with respect to $p$ is often intractable and we must resort to a Monte Carlo estimate.

\begin{definition}[Empirical Loss]
    Let $\ell$ be a loss function and $h \in \Hcal$. Let $\Dcal = \{(x_i,y_i)\}_{i=1}^N$ be a dataset drawn from $p$. Then, the \textbf{empirical loss} of the hypothesis $h$ (with respect to $\Dcal$) is
    \begin{equation}
        L_{\Dcal}(h) := \frac{1}{N} \sum_{i=1}^N \ell(y_i,h(x_i)).
    \end{equation}
\end{definition}

\begin{definition}[Learning Algorithm]
    A \textit{learning algorithm} $\Acal$ is an algorithm which takes a dataset $\Dcal$ as input and outputs a hypothesis $h \in \Hcal$.
\end{definition}

Abstracting away the details, the common learning algorithm for supervised deep learning is to take a neural network hypothesis class with a large number of layers $L$ (e.g., 50 layers in the case of the standard ResNet-50 architecture \cite{he2016deep} for classification on ImageNet \cite{russakovsky2015imagenet}) and minimize the empirical loss on a dataset $\Dcal$ using a variant of stochastic gradient descent (SGD).

The Probably Approximately Correct (PAC) Learning framework \cite{valiant1984theory} for statistical learning theory provides \textit{generalization bounds} for learning algorithms. That is, there exist upper bounds on the \textit{generalization gap}
\begin{equation}\label{eq:generalization_gap}
    |L(h) - L_{\Dcal}(h)|
\end{equation}
which hold with high probability. This captures (with high probability) the worst-case performance of a hypothesis $h$ on unseen data from the distribution $p$ given its performance on the training data $\Dcal$. Generalization bounds increase as hypothesis class complexity increases and decrease as training dataset size increases. The former concerns the "expressiveness" of the function class $\Hcal$, and is often captured by the Vapnik-Chervonenkis (VC) dimension or the Rademacher complexity \cite{mohri2018foundations,shalev2014understanding}. 

\begin{remark}
    Generalization bounds in statistical learning theory work under the assumptions outlined in Remark \ref{rmk:learning_theory}, rather than our setup of approximating a probability distribution. The loss function considered in statistical learning theory is the 0-1 loss 
    \begin{equation}\label{eq:zero_one_loss}
        \ell(y,\yhat) = \1(y \neq \yhat).
    \end{equation}
    Hence, the expected loss in this setting is the probability that a hypothesis misclassifies an input.
\end{remark}

\textbf{Unsupervised Learning.} There is no formalism akin to the one above for unsupervised learning, since the notion of learning the structure of the input space from unlabelled data is open-ended. 

\begin{definition}[Unlabelled Dataset]
    Let $\Xcal$ be an input space and let $p_x$ be a probability distribution over $\Xcal$. Then, a collection $\Dcal_x := \{x_i\}_{i=1}^N$ of samples drawn i.i.d.\@ from $p_x$ is said to be an \textit{unlabelled dataset} drawn from $p_x$.
\end{definition}

Unlabelled datasets are far easier to obtain than labelled datasets. While the latter often requires manual annotation (e.g., a human looking through images of handwritten digits and labelling each one), the former is readily available on the Internet. A plethora of text and image data can be extracted from social media, news websites, and online encyclopedias. A natural machine learning pipeline is then to first perform unsupervised learning on a large quantity of unlabelled data (with the aim of learning features), followed by supervised learning on whatever labelled data is available for the downstream task of interest.

Broadly speaking, the goal of unsupervised learning is to learn how data in the input space $\Xcal$ is distributed \cite{mohri2018foundations,goodfellow2016deep}. For example, in the case of clustering, we wish to group nearby (according to some notion of distance or similarity) samples \cite{ng2001spectral}. In dimensionality reduction (covered in detail in Chapter \ref{chapter:dim_reduction}), we wish to identify the subspace of $\Xcal$ on which data concentrates in order to produce low-dimensional representations \cite{pearson1901on,hotelling1933analysis,tenenbaum2000global,roweis2000nonlinear,belkin2003laplacian}. In general, such a subspace is nonlinear, which motivates \textit{manifold learning} approaches such as ISOMAP, LLE, and Laplacian Eigenmaps. 

In the machine learning context, a manifold is informally understood as "a connected set of points that can be approximated well by considering only a small number of degrees of freedom, or dimensions, embedded in a higher-dimensional space" \cite{goodfellow2016deep}. The \textit{manifold hypothesis} states that real-world data is concentrated on a low-dimensional manifold embedded in much higher-dimensional Euclidean space. The authors of \cite{goodfellow2016deep} outline two arguments in favour of the manifold hypothesis: (i) we are unlikely to obtain real-world text and image data by sampling uniformly over characters and pixel values respectively, and (ii) small perturbations (e.g., changing the brightness, cropping, translating) of real-world data would still be considered real-world data. This second point of particular importance for manifold learning methods, which specify a graph structure over the inputs in $\Dcal_x$ whereby similar examples are linked by a weighted edge (which depends on the notion of distance or similarity being employed). Low-dimensional representations then arise from the spectral decomposition of the adjacency matrix (or a related matrix, such as the graph Laplacian).

\textbf{Self-Supervised Learning.} Much of the developments in self-supervised learning have taken place over the last 5 years as its empirical performance began to rival that of supervised learning \cite{oord2018representation,chen2020simple,devlin2019bert,brown2020language}. As a result, theoretical knowledge in SSL is limited compared to the more established supervised paradigm \cite{haochen2021provable}. In fact, there is no consensus definition of self-supervised learning. We adopt the one from \cite{balestriero2022contrastive}, which states that "[self-supervised learning] places itself in-between supervised and unsupervised learning as it does not require labels but does require knowledge of what makes some samples semantically close to others." The goal of self-supervised learning is then to leverage unlabelled data along with this additional knowledge to learn representations of input data in $\Xcal$. The learned representations should be such that semantically close (i.e., similar) inputs have representations which are close (as measured by some notion of similarity or distance in the feature space).

An example of self-supervised learning in action is the work \cite{mahajan2018exploring}. The authors train a neural network encoder on 3.5 billion images from Instagram to learn representations for image classification. More specifically, they take the features that result from training a neural network to classify an image according to its associated hashtag(s). These hashtags serve as additional knowledge which can be automatically extracted. The auxiliary task of hashtag classification is referred to as the \textit{pretext task} and learning a model to solve this task is referred to as \textit{pre-training}. 

 Suppose we are interested in solving a set of downstream classification tasks $\Tcal_1,\dots,\Tcal_m$ over some domain $\Xcal$. Each task $\Tcal_i$, $i \in [m]$, is a tuple containing labelled training and test sets $\Tcal_i = (\Dcal_i^{train}, \Dcal_i^{test})$. One approach is to follow the supervised learning paradigm where for each $i \in [m]$, we train a classifier $h_i$ on $\Dcal_i^{train}$. Though this is natural, it limits the amount of training data available for each task. We are ignoring the data in $\Dcal_j^{train}$ for all $j \neq i$ in spite of the fact that it is drawn from the same domain.

Instead, we make use of a larger unlabelled dataset $\Dcal^{pre}$ (which can also contain the $\Dcal_i^{train})$. During pre-training, we define a pretext task $\Tcal^{pre}$ on $\Dcal^{pre}$. In contrastive learning, the pretext task is to distinguish similar pairs of inputs from dissimilar ones. The notion of similarity between inputs is user-defined and serves as the additional knowledge. The task is solved by learning a neural network feature map $\phibf: \Xcal \to \R^d$ for the inputs and computing a similarity score (e.g., dot product) between features. For example, the contrastive pretext task in SimCLR (see Section \ref{section:contrastive_augmentation}) is to determine which transformed images correspond to the same base image. Two random croppings of the same image of a cat should have a high dot product in the feature space, while an image of a cat and an image of a dog should have a low dot product. 

Then, on the downstream tasks, we train classifiers $h_1(\cdot) = g_1(\phibf(\cdot)),\dots,h_m(\cdot) = g_m(\phibf(\cdot))$ on the respective datasets $\Dcal_1^{train},\dots,\Dcal_m^{train}$. If $\phibf$ is held fixed and only $g$ is learned, this process is called \textit{probing}. Moreover, if $g$ is a linear classifier (i.e., a one-layer neural network), we call this \textit{linear probing}. Alternatively, if $\phibf$ is not held fixed, then we call this \textit{fine-tuning}.

\section{Classification}\label{section:classification}

Following our setup of a supervised machine learning problem in the previous section, we introduce the \textit{logistic regression} approach to classification, which will feature prominently in our discussion of contrastive learning in Chapter \ref{chapter:contrastive}.

\textbf{Binary Classification.} Consider a classification problem with $C = 2$ and a labelled dataset $\Dcal = \{(x_i,y_i)\}_{i=1}^N$ drawn from a distribution $p$ over $\Xcal \times \Ycal$. Without loss of generality, assume $\Ycal = \{0,1\}$. We would like to model the conditional probability $p(y=1|x)$ for all $x \in \Xcal$ (or, at least, all $x$ in the support of the marginal $p_x$).

Let $N^{(0)}, N^{(1)}$ denote the number of samples in $\Dcal$ with label 0 and 1, respectively. For each $x \in \Xcal$, denote by $N_x$ the number of times $x$ appears in $\Dcal$, and by $N_x^{(0)}, N_x^{(1)}$ the number of times $x$ appears with label 0 and 1 respectively. Let $k = \frac{N^{(0)}}{N^{(1)}}$. Further define $p_1(x) := p(x|y=1)$ and $p_0(x) := p(x|y=0)$ for every $x \in \Xcal$.

We learn a parametric function $s_{\theta}: \Xcal \to \R$ that assigns an unnormalized likelihood of belonging to class 1 for each input. Here, the parameter vector $\theta$ resides in some parameter space $\Theta$. Given $x \in \Xcal$, a larger value of $s_{\theta}(x)$ indicates that, according to our model, $x$ is more likely to belong to class 1. Examples of such functions include the affine map $x \mapsto w^{\top}x + b$ or a neural network. In both of these examples, the parameters which comprise $\theta$ are the weights and biases. 

Ultimately, we are interested in the probability that an input $x \in \Xcal$ was sampled from $p_1$. Thus, to convert $s_{\theta}(x)$ to a probability, we apply the \textit{sigmoid function}.

\begin{definition}[Sigmoid Function]
    The sigmoid function is the map $\sigma: \R \to \R$ defined by
    \begin{equation}
        \sigma(z) = \frac{1}{1+e^{-z}}.
    \end{equation}
\end{definition}

\begin{prop}[Properties of the Sigmoid Function]
    The following are useful properties of the sigmoid function $\sigma$.
    \begin{itemize}[topsep=0pt]
        \item[(1)] The range of $\sigma$ is the open interval $(0,1)$.
        \item[(2)] $\sigma(0) = \frac{1}{2}$.
        \item[(3)] The derivative of the sigmoid is $\sigma'(z) = \sigma(z)\big(1-\sigma(z)\big)$. In particular, $\sigma$ is a strictly increasing function over its domain.
    \end{itemize}
\end{prop}

A plot of the sigmoid function is provided in Figure \ref{fig:sigmoid}. From the properties above, we can interpret the output of our model in a straightforward fashion. Define 
\begin{equation}
    p_{\theta}(y=1|x) := \sigma\big(s_{\theta}(x)\big).
\end{equation}
$p_{\theta}(y=1|x)$ is our model's estimate of $p(y=1|x)$. Hence, our hypothesis class is
\begin{equation}
    \Hcal = \{\sigma(s_{\theta}(\cdot)): \theta \in \Theta\}.
\end{equation}


\begin{figure}
    \centering
    \includegraphics[width=0.5\linewidth]{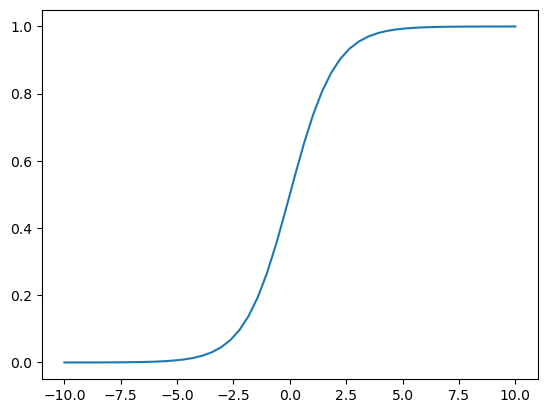}
    \caption{Plot of the sigmoid function.}
    \label{fig:sigmoid}
\end{figure}

A natural objective to learn $\theta$ is to maximize the likelihood $\Lcal$ of the data:
\begin{equation}\label{eq:binary_xent}
    \Lcal(\theta) = \prod_{i=1}^N \sigma\big(s_{\theta}(x_i)\big)^{\1(y_i=1)} \big(1-\sigma\big(s_{\theta}(x_i)\big)\big)^{\1(y_i=0)}.
\end{equation}
For computational reasons, we consider the equivalent objective of minimizing the negative log-likelihood, which we treat as the empirical loss:
\begin{equation}\label{eq:nll}
    -\log \Lcal(\theta) = \sum_{i=1}^N -\1(y_i=1) \log \sigma\big(s_{\theta}(x_i)\big) - \1(y_i=0) \log \big(1-\sigma\big(s_{\theta}(x_i)\big)\big).
\end{equation}
This is the empirical loss associated with the binary cross-entropy loss (\ref{eq:xent}).

We are interested in finding an optimal value of $\sigma(s_{\theta}(x))$ for each input $x \in \Xcal$. Since cross-entropy is convex in $\sigma(s_{\theta}(x))$, it suffices to find a root of its derivative with respect to this argument. Re-writing
\begin{equation}
    -\log \Lcal(\theta) = -\sum_{x \in \Xcal} N_x^{(1)} \log \sigma\big(s_{\theta}(x)\big) - N_x^{(0)} \log\big(1-\sigma\big(s_{\theta}(x)\big)\big),
\end{equation}
we may minimize term-by-term. For fixed $x \in \Xcal$, we have the derivative
\begin{equation}
    -\frac{N_x^{(1)}}{\sigma\big(s_{\theta}(x)\big)} + \frac{N_x^{(0)}}{1-\sigma\big(s_{\theta}(x)\big)}.
\end{equation}
Setting this to zero, we obtain
\begin{equation}
    \sigma\big(s_{\theta}(x)\big) = \frac{N_x^{(1)}}{N_x^{(0)}+N_x^{(1)}} = \frac{N_x^{(1)}}{N_x}.
\end{equation}

\begin{definition}[Empirical Probability]
    In the current setting, define the \textit{empirical (conditional) probabilities}
    \begin{equation}
        \phat(y=1|x) = \frac{N_x^{(1)}}{N_x} \hs\hs \phat(y=0|x) = \frac{N_x^{(0)}}{N_x}.
    \end{equation}
    for all $x \in \Xcal$.
\end{definition}

Now, we introduce the notion of a proper scoring function, which allows us to generalize the results we will develop to a broader class of loss functions.

\begin{definition}[Proper Scoring Function]
    An empirical loss $L_{\Dcal}: \Hcal \to \R$ is said to be a proper scoring function if its minimizer $h^*$ is such that
    \begin{equation}
        h^*(x) = \phat(y=1|x)
    \end{equation}
    for all $x \in \Xcal$.
\end{definition}
Clearly, the empirical loss obtained from the cross-entropy loss is a proper scoring function.

\begin{theorem}\label{thm:shifted_pmi}
    Let $L_{\Dcal}$ be a proper scoring function for our problem. Let $\theta^*$ be a minimizer of $L_{\Dcal}$. Then, for every $x \in \Xcal$,
    \begin{equation}\label{eq:shifted_pmi}
        s_{\theta^*}(x) = \log\bigg(\frac{\phat(y=1|x)}{\phat(y=0|x)}\bigg) = \log \bigg(\frac{\phat_1(x)}{\phat_0(x)}\bigg) - \log k,
    \end{equation}
    where $\log$ denotes the natural logarithm.
\end{theorem}
\begin{proof}
    Let $x \in \Xcal$. Since $\theta^*$ minimizes $L_{\Dcal}$, we have
    \begin{align*}
        &\sigma\big(s_{\theta^*}(x)\big) = \phat(y=1|x) \\
        \iff &\frac{1}{1+e^{-s_{\theta^*}(x)}} = \phat(y=1|x) \\
        \iff & 1+e^{-s_{\theta^*}(x)} = \frac{1}{\phat(y=1|x)} \\
        \iff &e^{-s_{\theta^*}(x)} = \frac{1}{\phat(y=1|x)} - 1 = \frac{1-\phat(y=1|x)}{\phat(y=1|x)} = \frac{\phat(y=0|x)}{\phat(y=1|x)} \\
        \iff &s_{\theta^*}(x) = \log\bigg(\frac{\phat(y=1|x)}{\phat(y=0|x)}\bigg).
    \end{align*}
    By Bayes Rule, we have
    \begin{equation*}
        \phat(y=1|x) = \frac{\phat(x|y=1)\phat(y=1)}{\phat(x)} = \frac{\phat_1(x)\phat(y=1)}{\phat(x)}
    \end{equation*}
    and
    \begin{equation*}
        \phat(y=0|x) = \frac{\phat(x|y=0)\phat(y=0)}{\phat(x)} = \frac{\phat_0(x)\phat(y=0)}{\phat(x)}.
    \end{equation*}
    Therefore,
    \begin{equation*}
        \log\bigg(\frac{\phat(y=1|x)}{\phat(y=0|x)}\bigg) = \log\bigg(\frac{\phat_1(x)}{\phat_0(x)}\bigg) - \log\bigg(\frac{N^{(0)}}{N^{(1)}}\bigg).
    \end{equation*}
\end{proof}

Using the Law of Large Numbers, we can also provide an asymptotic statement (i.e., as the dataset size tends to infinity).
\begin{corollary}\label{cor:shifted_pmi_limit}
    Let $\{(x_n,y_n)\}_{n=1}^{\infty}$ be a sequence of i.i.d.\@ samples from $p$. For every $N \in \N$, let $\Dcal_N = \{(x_n,y_n)\}_{n=1}^N$ and let $\theta_N^*$ be the minimizer of $L_{\Dcal_N}$. Further let $\kappa = \frac{p(y=0)}{p(y=1)}$. Then, for all $x \in \Xcal$,
    \begin{equation}\label{eq:shifted_pmi_limit}
        \lim_{N \to \infty} s_{\theta^*_N}(x) = \log\bigg(\frac{p_1(x)}{p_0(x)}\bigg) - \log \kappa
    \end{equation}
   almost surely.
\end{corollary}
\begin{proof}
    The result follows immediately from Theorem \ref{thm:shifted_pmi} and the Strong Law of Large Numbers.
\end{proof}

Let us take a closer look at Equation (\ref{eq:shifted_pmi_limit}). If $\kappa = 1$, then we can disregard the $\log \kappa$ term. In this case, for an input $x \in \Xcal$, in the infinite sample limit, our model will conclude that $x$ was more likely to be sampled from $p_1$ if and only if $p_1(x) > p_0(x)$. This is a natural decision criterion. However, when $k \neq 1$ (which is commonly the case in contrastive learning), this will not be the case. For instance, when $k > 1$, it is possible that $p_1(x) > p_0(x)$ but our model deems it more likely that $x$ was drawn from $p_0$. This occurs because the label $y=0$ has a higher prior probability.

It is possible to adjust our model to incorporate the prior into our activation function rather than into $s$. It is sufficient to modify the sigmoid activation function \cite{gutmann2012noise}.
\begin{definition}[$k$-sigmoid]
    Define $\sigma_k: \R \to \R$ such that
    \begin{equation}\label{eq:k_sigmoid}
        \sigma_k(z) = \frac{1}{1+ke^{-z}}.
    \end{equation}
\end{definition}

We plot this function for various settings of $k$ in Figure \ref{fig:k_sigmoid}. Notice that if there are more negative samples than positive samples ($k > 1$), then $\sigma_k$ attains the value 1/2 at a point strictly larger than zero, and vice versa.

\begin{figure}
    \centering
    \includegraphics[width=0.5\linewidth]{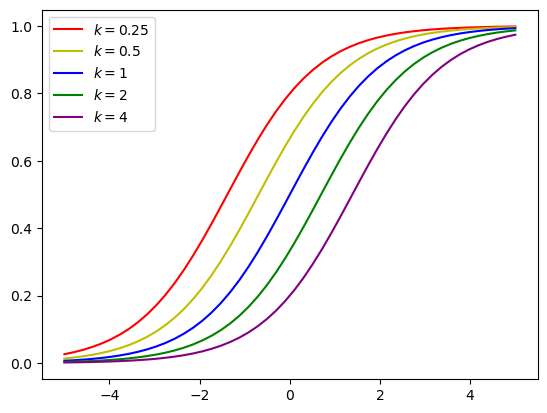}
    \caption{The modified sigmoid activation function (\ref{eq:k_sigmoid}) for different values of $k$.}
    \label{fig:k_sigmoid}
\end{figure}

\begin{corollary}\label{cor:centred_pmi}
    Let $L_{\Dcal}$ be a proper scoring function for our problem. Suppose that we modify our model to use the activation function $\sigma_k$ rather than $\sigma$. Let $\theta^*$ be a minimizer of $L_{\Dcal}$. Then, for every $x \in \Xcal$,
    \begin{equation}
        s_{\theta^*}(x) = \log\bigg(\frac{\phat_1(x)}{\phat_0(x)}\bigg).
    \end{equation}
\end{corollary}
\begin{proof}
    The steps of the proof are the same as in Theorem \ref{thm:shifted_pmi}.
\end{proof}

\begin{corollary}\label{cor:centred_pmi_limit}
    Suppose we are in the same setting as Corollary \ref{cor:shifted_pmi_limit}, but the activation function is replaced by $\sigma_k$. Then, for all $x \in \Xcal$,
    \begin{equation}\label{eq:centred_pmi_limit}
        \lim_{N \to \infty} s_{\theta^*_N}(x) = \log\bigg(\frac{p_1(x)}{p_0(x)}\bigg)
    \end{equation}
    almost surely.
\end{corollary}

Thus, in learning how to distinguish between $p_0$ and $p_1$, we are also learning the log ratio
\begin{equation}
    \log\bigg(\frac{p_1(x)}{p_0(x)}\bigg)
\end{equation}
in the limit of infinite samples. This quantity is useful beyond the abstract task that we have presented in this section. For instance, if $p_0$ is a known distribution and we have learned the log ratio, then we have a means of estimating $p_1$. In other words, it is possible to model an unknown distribution by comparing it with a known distribution $p_0$ and solving the classification problem that we have outlined above. This is the idea underlying noise-contrastive estimation \cite{gutmann2010noise,gutmann2012noise}, which we discuss in Section \ref{section:nce}.

\textbf{Multiclass Classification.} In the case $C > 2$, we have a similar setup to the binary case. Once again, we have a dataset $\Dcal = \{(x_i,y_i)\}_{i=1}^N$ drawn i.i.d.\@ from a distribution $p$ over $\Xcal \times \Ycal$. Assume without loss of generality that $\Ycal = [C]$. For every class $c \in \Ycal$, define $N^{(c)}$ to be the number of samples in $\Dcal$ with label $c$. Further, for every $x \in \Xcal$, define $N_x^{(c)}$ and $p_c(x)$ similarly to the binary case. 

We once again learn a parametric function $s_{\theta}: \Xcal \to \R^C$ which represents the "scores" our model assigns to each class. To enforce the constraint $\sum_{c=1}^C p_{\theta}(y=c|x) = 1$, we convert scores to probabilities using the softmax function, an extension of the sigmoid to the multiclass case.

\begin{definition}[Softmax Function]
    Fix $j \in \{1,\dots,C\}$. The \textit{softmax} function is the map $\sm: \R^C \to \R^C$ defined (coordinate-wise) by
    \begin{equation}
        \sm(z)_c = \frac{\exp(z_c)}{\sum_{c'=1}^C \exp(z_{c'})}.
    \end{equation}
\end{definition}

We thus define, for all $x \in \Xcal$ and $c \in \{1,\dots,C\}$,
\begin{equation}
    p_{\theta}(y=c|x) = \sm\big(s_{\theta}(x)\big)_c.
\end{equation}
The data likelihood is then
\begin{equation}
    \Lcal(\theta) = \prod_{i=1}^N \prod_{c=1}^C \big(p_{\theta}(y=c|x_i)\big)^{\1(y_i=c)}.
\end{equation}
We may then write the negative log-likelihood as
\begin{align}
    -\log \Lcal(\theta) &= \sum_{i=1}^N \sum_{c=1}^C -\1(y_i=c) \log p_{\theta}(y=c|x_i) \\
    &= \sum_{x \in \Xcal} \sum_{c=1}^C -N_x^{(c)} \log p_{\theta}(y=c|x).
\end{align}
Minimizing this reduces to solving the following optimization problem over $(p_{\theta}(y=1|x),\dots,p_{\theta}(y=C|x))$ for each $x \in \Xcal$:
\begin{align}\label{eq:constrained_optim}
    \text{Minimize} \hs &f\big(p_{\theta}(y=1|x),\dots,p_{\theta}(y=C|x)\big) := \sum_{c=1}^C -N_x^{(c)} \log p_{\theta}(y=c|x) \\
    \text{subject to} \hs &g\big(p_{\theta}(y=1|x),\dots,p_{\theta}(y=C|x)\big) := \sum_{c=1}^C p(y=c|x) = 1.
\end{align}

\begin{prop}
    The above problem is a convex optimization problem.
\end{prop}
\begin{proof}
    The lone constraint is affine, and the constraint set is therefore convex. Moreover,
    \begin{equation}
        \nabla f = \bigg(-\frac{N_x^{(1)}}{p_{\theta}(y=1|x)},\dots,-\frac{N_x^{(C)}}{p_{\theta}(y=C|x)}\bigg)
    \end{equation}
    and
    \begin{equation}
        \nabla^2 f = \diag\bigg(\frac{N_x^{(1)}}{\big(p_{\theta}(y=1|x)\big)^2},\dots,\frac{N_x^{(C)}}{\big(p_{\theta}(y=C|x)\big)^2}\bigg),
    \end{equation}
    which is positive semidefinite. 
\end{proof}

We can solve the minimization problem via the method of Lagrange multipliers. By convexity, any solution we find is guaranteed to be a global minimum. We have the system of equations
\begin{equation}
    \forall \, c \in [C] \hs\hs \frac{N_x^{(c)}}{p_{\theta}(y=c|x)} = \lambda.
\end{equation}
By the sum-to-one constraint, we have
\begin{equation}
    \sum_{c=1}^C N_x^{(c)} = \lambda \hs\hs \Longrightarrow  \hs\hs N_x = \lambda.
\end{equation}
Therefore, we obtain the maximum likelihood estimates
\begin{equation}
    \forall \, C \in \{1,\dots,C\} \hs\hs \phat(y=c|x) := p_{\theta^*}(y=c|x) = \frac{N_x^{(c)}}{N_x}.
\end{equation}

As in the case of binary classification, we say that an empirical loss $L_{\Dcal}$ is a proper scoring rule if the MLE is an optimum.

%% file: Chapters/rkhs.tex
\chapter{Reproducing Kernel Hilbert Spaces}\label{chapter:rkhs}

The theory of reproducing kernel Hilbert spaces (RKHSs) was rigorously established in the 1950s, most notably with the seminal work \cite{aronszajn1950theory}, and in recent decades kernel methods for machine learning have been well-studied \cite{scholkopf2018learning}. We are particularly interested in positive semidefinite kernels, as they specify a notion of pairwise similarity on the input space $\Xcal$ and implicitly define features.

In this section, we summarize the main results concerning kernels that will be relevant to our study, with \cite{wainwright2019high} and \cite{mohri2018foundations} being our main points of reference.

\section{Positive Semidefinite Kernels}\label{section:psd_kernels}

\begin{definition}[Positive Semidefinite Kernel]
    Let $K: \Xcal \times \Xcal \to \R$ be a symmetric function (called a \textbf{kernel}). If for every $n \in \N$ and $\{x_i\}_{i=1}^n \subseteq \Xcal$, the matrix $\Gbf \in \R^{n \times n}$ (called a \textbf{Gram matrix}) with entries $G_{ij} = K(x_i,x_j)$ is positive semidefinite, then $K$ is said to be a \textbf{positive semidefinite (PSD) kernel}.
\end{definition}

\begin{prop}\label{prop:inner_product_kernel}
    Let $\Vbb$ be an inner product space and suppose $\Xcal \subseteq \Vbb$. Then, the kernel $K: \Xcal \times \Xcal \to \R$ defined by $K(x,z) = \ev{x,z}_{\Vbb}$ is PSD.
\end{prop}
\begin{proof}
    Let $n \in \N$, $\{\xbf_i\}_{i=1}^n \subseteq \Xcal$, and $\Gbf \in \R^{n \times n}$ be the associated Gram matrix, i.e., $G_{ij} = \ev{\xbf_i,\xbf_j}_{\Vbb}$. Let $\abf \in \R^n$. Then, by bilinearity of $\ev{\cdot,\cdot}_{\Vbb}$,
    \begin{equation}
        \abf^T \Gbf \abf = \sum_{i=1}^n \sum_{j=1}^n a_i a_j \ev{\xbf_i,\xbf_j}_{\Vbb} = \bigg\langle \sum_{i=1}^n a_i \xbf_i, \sum_{j=1}^n a_j \xbf_j\bigg\rangle_{\Vbb} \geq 0.
    \end{equation}
\end{proof}

We outline examples of positive semidefinite (PSD) kernels provided in \cite{wainwright2019high} and \cite{mohri2018foundations}.

\begin{example}[Linear Kernel]
    Let $\Xcal = \R^d$ and for all $\xbf,\zbf \in \Xcal$, define $K(\xbf,\zbf) := \xbf \cdot \zbf$, i.e., the Euclidean dot product. By the previous proposition, $K$ is a PSD kernel.
\end{example}

\begin{example}[Polynomial Kernel]
    Let $\Xcal = \R^d$ and for all $\xbf,\zbf \in \Xcal$, define $K(\xbf,\zbf) := (\xbf \cdot \zbf)^m$ for some $m \in \N$. The previous example covers the case $m = 1$. Suppose now that $m = 2$. Then,
    \begin{equation}
        \forall \, \xbf,\zbf \in \Xcal \hs\hs K(\xbf,\zbf) = (\xbf \cdot \zbf)^2 = \sum_{j=1}^d x_j^2 z_j^2 + 2 \sum_{i < j} x_i x_j z_i z_j.
    \end{equation}
    This kernel in fact corresponds to a Euclidean dot product in the higher-dimensional space $\R^D$ with $D = d + \binom{d}{2}$. To see this, define the mapping $\Phibf: \R^d \to \R^D$ such that, for $\xbf \in \R^d$ and for each $j \in [d]$, $\Phibf(\xbf)$ contains the entries $x_j^2$ and $\sqrt{2} x_i x_j$ for all $i < j$.

    First, we notice that $K(\xbf, \zbf) = \Phi(\xbf) \cdot \Phi(\zbf)$ for all $\xbf,\zbf \in \Xcal$, which implies that $K$ is PSD. Second, we observe that the polynomial kernel implicitly defines a higher-dimensional representation of the elements in $\Xcal$. This finding, which is not exclusive to the polynomial kernel, is the basis of kernel methods in machine learning, such as the support vector machine \cite{boser1992training} and kernel principal component analysis \cite{scholkopf1998nonlinear}. 
\end{example}

\begin{example}[Gaussian Kernel] 
    Let $\Xcal = \R^d$ and for all $\xbf,\zbf \in \Xcal$, define $K(\xbf,\zbf) := \exp(-\frac{1}{2\sigma^2}||\xbf-\zbf||_2^2)$ for some $\sigma^2 > 0$. Notice
    \begin{align}
        \exp\bigg(-\frac{1}{2\sigma^2} ||\xbf-\zbf||_2^2\bigg) &= \exp\bigg(-\frac{1}{2\sigma^2} ||\xbf||^2\bigg)\exp\bigg(-\frac{1}{2\sigma^2} ||\zbf||^2\bigg)\exp\bigg(\frac{\xbf \cdot \zbf}{\sigma^2}\bigg) \\
        &= \underbrace{\exp\bigg(-\frac{1}{2\sigma^2}||\xbf||^2\bigg)}_{=: c(\xbf)}\underbrace{\exp\bigg(-\frac{1}{2\sigma^2}||\zbf||^2\bigg)}_{=:c(\zbf)}\sum_{k=0}^{\infty} \frac{(\xbf \cdot \zbf)^k}{\sigma^{2k}k!} \\
        &= \sum_{k=0}^{\infty} \frac{(c(\xbf)^{\frac{1}{k}}\xbf \cdot c(\zbf)^{\frac{1}{k}}\zbf)^k}{\sigma^{2k}k!}.
    \end{align}
    Thus, we may view the Gaussian kernel as an infinite sum of scaled polynomial kernels. Recalling our finding in the previous example, we see the Gaussian kernel as implicitly computing an inner product in an infinite-dimensional space. 
    
    The proof that the Gaussian kernel is a PSD kernel relies on properties we will not cover here (e.g., the pointwise limit of a sequence of PSD kernels is a PSD kernel). For details, see Exercise 12.19 in \cite{wainwright2019high}.
\end{example}

The following kernel will be relevant to our study of word embeddings in Section \ref{section:word_embeddings}. We proceed similarly to Proposition 2.19 in \cite{scholkopf2018learning} to show that it is PSD.

\begin{example}[Exponentiated Pointwise Mutual Information]\label{ex:exp_pmi}
    Let $(\Omega, \Fcal, \Prob)$ be a probability space. Define the kernel
    \begin{equation}
        K(A,B) := \frac{\Prob(A \cap B)}{\Prob(A) \Prob(B)} \hs\hs \forall \, A,B \in \Fcal.
    \end{equation}
    For every event $A \in \Fcal$, define $\phi(A) = \frac{\1_A}{\Prob(A)} \in L^2(\Omega;\Prob)$. Then, for any two events $A,B \in \Fcal$,
    \begin{equation}
        \bigg \langle \frac{\1_A}{\Prob(A)}, \frac{\1_B}{\Prob(B)} \bigg \rangle_{L^2(\Omega;\Prob)} = \frac{1}{\Prob(A) \Prob(B)}\int_{\Omega} \1_A \1_B \, d\Prob = \frac{\Prob(A \cap B)}{\Prob(A) \Prob(B)}.
    \end{equation}
    We appeal to Proposition \ref{prop:inner_product_kernel} to conclude that $K$ is PSD.
\end{example}

In Chapter \ref{chapter:contrastive}, we will frequently encounter exponentials and logarithms of kernels. In preparation for this, we note the following useful result. 
\begin{prop}[Exponential of PSD Kernel is PSD, \cite{scholkopf2018learning}]\label{prop:exp_kernel}
    Let $K: \Xcal \times \Xcal \to \R$ be a PSD kernel. Then, the kernel $\tilde{K}: \Xcal \times \Xcal \to \R$ defined by $\tilde{K}(x,z) = \exp\big(K(x,z)\big)$ is also PSD.
\end{prop}



\section{Features from Kernels}\label{section:features_from_kernels}

We now formalize the notion that positive semidefinite kernels implicitly capture the geometry of high-dimensional spaces.

\begin{definition}[Reproducing Kernel]
    Let $\Hbb$ be a Hilbert space of real-valued functions on $\Xcal$ and $K: \Xcal \times \Xcal \to \R$ be a PSD kernel such that the partial evaluations $K(\cdot,x): \Xcal \to \R$ are in $\Hbb$ for every $x \in \Xcal$. Then, $K$ called a \textbf{reproducing kernel} if it satisfies
    \begin{equation}\label{eq:reproducing_property}
        \ev{f,K(\cdot,x)}_{\Hbb} = f(x) \hs\hs \forall \, f \in \Hbb \hs \forall \, x \in \Xcal.
    \end{equation}
    We call (\ref{eq:reproducing_property}) the \textbf{reproducing property}.
\end{definition}

\begin{definition}[Reproducing Kernel Hilbert Space]
    A \textbf{reproducing kernel Hilbert space (RKHS)} $\Hbb$ is a Hilbert space of real-valued functions on $\Xcal$ such that for all $x \in \Xcal$, the linear evaluation functional $L_x: \Hcal \to \R$ defined by $L_x(f) = f(x)$ is bounded.
\end{definition}

\begin{theorem}[Moore-Aronszajn]\label{thm:moore-aronszajn}
    Given a PSD kernel $K: \Xcal \times \Xcal \to \R$, there exists a unique RKHS $\Hbb$ such that $K$ satisfies the reproducing property. Conversely, for every RKHS $\Hbb$, there exists a unique reproducing kernel $K$.
\end{theorem}

\begin{proof}
    We follow the proofs outlined in \cite{wainwright2019high} and \cite{mohri2018foundations}, filling in the details where necessary.
    $(\Rightarrow)$ Let $K$ be a PSD kernel. We construct the associated RKHS $\Hbb$. We construct the vector space
    \begin{equation}
        \Hbb_0 := \bigg\{\sum_{i=1}^m a_i K(\cdot,x_i): m \in \N, x_i \in \Xcal, a_i \in \R\bigg\},
    \end{equation}
    the set of all linear combinations of partial kernel evaluations.

    We define an inner product $\ev{\cdot,\cdot}_{\Hbb_0}$ in the following way. For all $f,g \in \Hbb_0$ with $f = \sum_{i=1}^m a_i K(\cdot,x_i)$ and $g = \sum_{j=1}^{m'} b_j K(\cdot,z_j)$ with
    \begin{equation}
        \ev{f,g}_{\Hbb_0} = \sum_{i=1}^m \sum_{j=1}^{m'} a_i b_j K(x_i,z_j) = \sum_{i=1}^m a_i g(x_i) = \sum_{j=1}^{m'} b_i f(z_j).
    \end{equation}
    Suppose now that we can also represent $f$ by $f = \sum_{i=1}^{\tilde{m}} c_i K(\cdot,\tilde{x}_i)$. Then we still have 
    \begin{equation}
        \ev{f,g}_{\Hbb_0} = \sum_{j=1}^{m'} b_i f(z_j).
    \end{equation}
    Similarly, we may argue that changing the representation of $g$ does not change the result, which confirms that $\ev{\cdot,\cdot}_{\Hbb_0}$ is well-defined. Symmetry and bilinearity is straightforward, and for all $f \in \Hbb_0$, we have
    \begin{equation}\label{eq:self_inner_product}
        \ev{f,f}_{\Hbb_0} = \sum_{i=1}^m \sum_{j=1}^m a_i a_j K(x_i,x_j) \geq 0
    \end{equation}
    since $K$ is PSD. It remains to verify that if the left-hand side of (\ref{eq:self_inner_product}) is zero, then $f = 0$. To see this, let $a \in \R$, $x \in \Xcal$. We have
    \begin{equation*}
        0 \leq \bigg|\bigg|a K(\cdot, x) + \sum_{i=1}^n a_i K(\cdot,x_i)\bigg|\bigg|_{\Hbb}^2 = a^2 K(x,x) + 2a \sum_{i=1}^n a_i K(x,x_i) + \ev{f,f}_{\Hbb_0}.
    \end{equation*}
    Since $K(x,x) \geq 0$ and the above holds for any choice of $a \in \R$, we must have $f = 0$. 

    Moreover, the reproducing property holds for $\ev{\cdot,\cdot}_{\Hbb_0}$. For $f = \sum_{i=1}^m a_i K(\cdot,x_i) \in \Hbb_0$ and $x \in \Xcal$, 
    \begin{equation}
        \ev{f,K(\cdot,x)} = \sum_{i=1}^m a_i K(x_i,x) = f(x).
    \end{equation}

    Now, we complete $\Hbb_0$ to a Hilbert space. Note that if $\{f_n\}_{n=1}^{\infty}$ is a Cauchy sequence in $\Hbb_0$, then $\{f_n(x)\}_{n=1}^{\infty}$ is Cauchy in $\R$ for all $x \in \Xcal$. Indeed, by Cauchy-Schwarz,
    \begin{equation*}
        |f_n(x) - f_m(x)|^2 = \big|\ev{f_n,K(\cdot,x)}-\ev{f_m,K(\cdot,x)}\big|^2 = \big|\ev{f_n-f_m,K(\cdot,x)}\big|^2 \leq ||f_n-f_m||_{\Hbb_0}^2 K(x,x)
    \end{equation*}
    for all $m,n \in \N$. Then, we may define $\Hbb$ to be the completion of $\Hbb_0$ by pointwise limits of Cauchy sequences and define the norm $||f||_{\Hbb} = \lim_{n \to \infty} ||f_n||_{\Hbb_0}$ for $f \in \Hbb$ such that $\lim_{n \to \infty} f_n(x) = f(x)$ for all $x \in \Xcal$.

    To check, that the norm is well-defined, let $\{g_n\}_{n=1}^{\infty}$ be a Cauchy sequence in $\Hbb_0$ such that $\lim_{n \to \infty} g_n(x) = 0$ for all $x \in \Xcal$. Suppose that $\{||g_n||_{\Hbb_0}\}_{n=1}^{\infty}$ does not converge to zero. Taking a subsequence if necessary, we have that $\lim_{n \to \infty} ||g_n||_{\Hbb_0} = 2\eps$ for some $\eps > 0$. Since we may write $g_m = \sum_{i=1}^{M_m} a_i K(\cdot,x_i)$, we have
    \begin{equation*}
        \ev{g_m,g_n}_{\Hbb_0} = \sum_{i=1}^{M_m} a_i g_n(x_i) \to 0
    \end{equation*}
    as $n \to \infty$ by pointwise convergence. Then, for $n,m$ sufficiently large, we simultaneously have $||g_n||_{\Hbb_0} > \eps$, $||g_m||_{\Hbb_0} > \eps$, $|\ev{g_m,g_n}_{\Hbb_0}| < \eps/2$, and $||g_n-g_m||_{\Hbb_0} < \eps/2$. But then,
    \begin{equation}
        ||g_n-g_m||_{\Hbb_0}^2 = ||g_n||_{\Hbb_0}^2 + ||g_m||_{\Hbb_0}^2 - 2\ev{g_m,g_n}_{\Hbb_0} \geq \eps + \eps - \eps = \eps
    \end{equation}
    yields a contradiction. Hence, our norm on $\Hbb$ is well-defined.

    The inner product on $\Hbb$ is then obtained by polarization:
    \begin{equation}\label{eq:inner_product_polarization}
        \ev{f,g}_{\Hbb} = \frac{1}{2}\big(||f+g||_{\Hbb}^2 - ||f||_{\Hbb}^2 - ||g||_{\Hbb}^2\big).
    \end{equation}
    We show that this inner product satisfies the reproducing property. Let $f \in \Hbb$ such that $f = \lim_{n \to \infty} f_n$ (pointwise) for some $\{f_n\}_{n=1}^{\infty} \subseteq \Hbb_0$ and let $x \in \Xcal$. We have
    \begin{equation*}
        ||K(\cdot,x)||_{\Hbb}^2 = \ev{K(\cdot,x), K(\cdot,x)}_{\Hbb_0} = K(x,x),
    \end{equation*}
    and
    \begin{align*}
        ||f + K(\cdot,x)||_{\Hbb}^2 &= \lim_{n \to \infty} ||f_n + K(\cdot,x)||^2_{\Hbb_0} \\
        &= \lim_{n \to \infty} \ev{f_n + K(\cdot,x), f_n + K(\cdot,x)}_{\Hbb_0} \\
        &= \lim_{n \to \infty} \bigg(||f_n||_{\Hbb_0}^2 + K(x,x) + 2f_n(x)\bigg) \\
        &= ||f||_{\Hbb}^2 + K(x,x) + 2f(x).
    \end{align*}
    From this and (\ref{eq:inner_product_polarization}), we obtain
    \begin{equation}
        \ev{f,K(\cdot,x)} = f(x),
    \end{equation}
    showing that the reproducing property is satisfied.

    Finally, let $L_x: \Hbb \to \R$ be the linear evaluation functional at $x \in \Xcal$. Then,
    \begin{equation}
        ||L_x(f)||^2 = |f(x)|^2 = \ev{f,K(\cdot,x)}^2 \leq \ev{f,f}K(x,x) = ||f||^2_{\Hbb_0} K(x,x).
    \end{equation}
    Thus, $L_x$ is a bounded linear operator.

    $(\Leftarrow)$ Let $\Hbb$ be an RKHS. Then, the linear evaluation functional $L_x: \Hbb \to \R$ is bounded for all $x \in \Xcal$. By the Riesz Representation Theorem, for each $x \in \Xcal$ there exists $R_x \in \Hbb$ such that
    \begin{equation}
        f(x) = L_x(f) = \ev{f,R_x} \hs \forall \, f \in \Hbb.
    \end{equation}
    Define $K: \Xcal \times \Xcal \to \R$ by
    \begin{equation}
        K(x,z) = \ev{R_x,R_z}_{\Hbb},
    \end{equation}
    which is PSD by Proposition \ref{prop:inner_product_kernel}. Moreover, for all $x,z \in \Xcal$,
    \begin{equation}
        K(z,x) = \ev{R_z,R_x}_{\Hbb} = R_x(z),
    \end{equation}
    which implies that $K(\cdot,x) = R_x$. Therefore, $K$ satisfies the reproducing property.

    To show uniqueness, suppose there exists another PSD kernel $K'$ which satisfies the reproducing property. Then, since $K$ also satisfies the reproducing property, we have
    \begin{equation}
        K(x,z) = \ev{K(\cdot,z), K'(\cdot,x)}_{\Hbb} = K'(x,z).
    \end{equation}
\end{proof}

\begin{remark}
    A key insight from the proof above is that every PSD kernel $K$ defines a feature map $\Phibf: \Xcal \to \Hbb$ given by $\Phibf(x) = K(\cdot,x)$. Evaluating the kernel for a pair of inputs $x,z \in \Xcal$ can be seen as computing an inner product in the (possibly infinite-dimensional) space $\Hbb$,
    \begin{equation}
        K(x,z) = \ev{K(\cdot,x), K(\cdot,z)}_{\Hbb} = \ev{\Phibf(x),\Phibf(z)}_{\Hbb}.
    \end{equation}
\end{remark}



A subsequent question of interest is whether kernels can also be used to produce low-dimensional embeddings of inputs in $\Xcal$. Mercer's Theorem provides an affirmative response for PSD kernels. We follow the presentation from \cite{wainwright2019high}.

\begin{theorem}[Mercer]\label{thm:mercer}
    Let $\mu$ be a nonnegative measure over a compact metric space $\Xcal$. Let $K$ be a continuous PSD kernel on $\Xcal \times \Xcal$ which satisfies
    \begin{equation}\label{eq:mercer_condition}
        \int_{\Xcal \times \Xcal} K(x,z) \, \mu(dx) \, \mu(dz) < \infty.
    \end{equation}
    Define the operator $T_K: L^2(\Xcal;\mu) \to L^2(\Xcal;\mu)$ by
    \begin{equation}
        T_K(f)(x) := \int_{\Xcal} K(x,z)f(z) \, \mu(dz).
    \end{equation}
    Then, there exist a sequence of eigenfunctions $\{\phi_j\}_{j=1}^{\infty}$ that form on orthonormal basis of $L^2(\Xcal;\mu)$ and nonnegative eigenvalues $\{\lambda_j\}_{j=1}^{\infty}$ such that
    \begin{equation}
        T_K(\phi_j) = \lambda_j\phi_j
    \end{equation}
    for all $j \in \N$. Moreover, $K$ has the expansion
    \begin{equation}\label{eq:eigenfunction_expansion}
        K(x,z) = \sum_{j=1}^{\infty} \lambda_j \phi_j(x)\phi_j(z),
    \end{equation}
    where the convergence holds absolutely and uniformly.
\end{theorem}

\begin{example}[Finite Input Space Case, \cite{wainwright2019high}]\label{ex:mercer_finite}
    Suppose that $\Xcal$ is finite and let $N = |\Xcal|$. Let $K: \Xcal \times \Xcal \to \R$ be a PSD kernel and $\Gbf$ be its associated $N \times N$ Gram matrix. Let $\mu$ be the counting measure on $\Xcal$. Then, the Mercer eigenvalues and eigenfunctions of $K$ with respect to $\mu$ are the eigenvalues and eigenvectors of the matrix $\Gbf$.
\end{example}

As pointed out in \cite{wainwright2019high}, Mercer's Theorem shows that a PSD kernel (under some assumptions) induces an embedding of $\Xcal$ into $\ell^2$, the space of square-summable sequences. Indeed, we may define a feature map $\Phibf: \Xcal \to \ell^2$ such that
\begin{equation}\label{eq:infinite_embedding}
    \Phibf(x) = \big(\sqrt{\lambda_1}\phi_1(x),\sqrt{\lambda_2}\phi_2(x),\dots\big).
\end{equation}
Then, by (\ref{eq:eigenfunction_expansion}), we have
\begin{equation}
    K(x,z) = \ev{\Phibf(x),\Phibf(z)}_{\ell^2(\N)}.
\end{equation}

To obtain low-dimensional features, we can truncate the vector (\ref{eq:infinite_embedding}) to a fixed dimension $d$. This idea will manifest in our discussion of the Nyström method in Section \ref{section:kernel_approx} and Neural Eigenmaps \cite{deng2022neural} in Section \ref{section:learning_mercer_eigenfunctions}.

%% file: Chapters/dim_reduction.tex
\chapter{Dimensionality Reduction}\label{chapter:dim_reduction}

While there is no agreed upon definition for what makes an adequate vector representation of data, it is commonly said that the goal of feature learning is to uncover the low-dimensional latent structure from which the observed data is generated \cite{bengio2013representation}. For example, a raw image contains hundreds or thousands of pixels. However, there may only be a small number of latent factors of variation which influence a classification decision, such as pose and lighting \cite{tenenbaum2000global}. These latent variables cannot immediately be gleaned from raw pixels, and therefore we desire a mapping from observed (input) space to latent (feature) space. 

There is a variety of dimensionality reduction methods which produce a mapping between observed and latent space. Linear methods, such as principal component analysis (PCA) \cite{pearson1901on,hotelling1933analysis} and multidimensional scaling (MDS) \cite{cox2000multidimensional,wang2012geometric}, assume that the observations and the latent representations are related by a linear map. We explore these in Section \ref{section:linear_dim_reduction}. The manifold learning approaches discussed in Section \ref{section:manifold_learning} do not make this assumption. They rely on a notion of similarity (or dissimilarity) between examples in the training set. In Section \ref{section:kernel_approx}, we explore methods which, given a kernel (i.e., a notion of similarity), aim to directly learn vector representations whose inner products approximate that kernel.

It is important to note that the notions of similarity employed by these dimensionality reduction methods do not incorporate any additional knowledge of what makes examples similar. Instead, they make use of standard functions that can be directly computed in the input space (e.g., Euclidean distance, the Gaussian kernel). Hence, we still consider them unsupervised methods. Throughout this chapter, we assume $\Xcal \subseteq \R^{n_0}$ and that we have a dataset $\Dcal_{\xbf} = \{\xbf_i\}_{i=1}^N$ sampled from a distribution $p_{\xbf}$ on $\Xcal$.

\section{Linear Methods}\label{section:linear_dim_reduction}

\textbf{Principal Component Analysis (PCA).} PCA learns a low-rank orthogonal projection of the data in the input space that maximally preserves variance. We provide an overview of the theory based on \cite{wainwright2019high,mohri2018foundations}.

We wish to solve the following minimization problem:
\begin{equation}\label{eq:original_pca_problem}
    \underset{\Pbf \in \Pcal_d^{n_0}}{\argmin} \,\, \E_{p_{\xbf}}\big[||\xbf-\Pbf\xbf||_2^2\big],
\end{equation}
where $\Pcal_d^{n_0} \subseteq \R^{n_0 \times n_0}$ denotes the set of orthogonal projection matrices of rank $d$. This objective encourages projected inputs to remain as close as possible to original inputs, with higher priority given to regions of the input space with higher probability under $p_{\xbf}$.

Re-writing the objective, we have
\begin{align*}
    \E\big[||\xbf-\Pbf\xbf||_2^2\big] &= \E\big[(\xbf-\Pbf\xbf)^{\top}(\xbf-\Pbf\xbf)\big] \\
    &= \E\big[||\xbf||_2^2\big] + \E\big[\xbf^{\top}\Pbf\xbf\big] - \E\big[\xbf^{\top}\Pbf^{\top}\xbf\big] + \E\big[\xbf^{\top} \Pbf^{\top}\Pbf\xbf\big] \\
    &= \E\big[||\xbf||_2^2\big] - \E\big[\xbf^{\top} \Pbf \xbf\big],
\end{align*}
where the last line follows from the fact that $\Pbf$ is symmetric and idempotent (since $\Pbf$ is an orthogonal projection matrix). Hence, our optimization problem simplifies to
\begin{equation}\label{eq:pca_problem}
    \underset{\Pbf \in \Pcal_d^{n_0}}{\argmax} \,\, \E\big[\xbf^{\top} \Pbf \xbf\big].
\end{equation}

\begin{theorem}
    The matrix $\Pbf = \Vbf \Vbf^{\top}$, where the columns of $\Vbf \in \R^{n_0 \times d}$ are the top $k$ eigenvectors (i.e., the eigenvectors associated with the $k$ largest eigenvalues) of the covariance matrix $\E[\xbf \xbf^{\top}]$, solves (\ref{eq:pca_problem}).
\end{theorem}
\begin{proof}
    For every $\Pbf \in \Pcal_d^{n_0}$, we may write $\Pbf = \Ubf \Ubf^{\top}$ for some $\Ubf \in \R^{n_0 \times d}$ with orthonormal columns. Then,
    \begin{equation}
        \E\big[\xbf^{\top} \Pbf \xbf\big] = \E\big[\xbf^{\top} \Ubf \Ubf^{\top}\xbf\big] = \E\big[||\Ubf^{\top}\xbf||_2^2\big].
    \end{equation}
    Let $\ubf_1,\dots,\ubf_d$ denote the columns of $\Ubf$. Then,
    \begin{equation}\label{eq:variance_maximization}
        \E\big[||\Ubf^{\top}x||_2^2\big] = \sum_{j=1}^d \E\big[\ev{\ubf_j,\xbf}^2\big] = \sum_{j=1}^d \ubf_j^{\top} \E[\xbf \xbf^{\top}]\ubf_j.
    \end{equation}
    We may write $\E[\xbf \xbf^{\top}] = \sum_{i=1}^{n_0} \lambda_i \vbf_i \vbf_i^{\top}$, where $\{\lambda_i\}_{i=1}^{n_0}$, $\{\vbf_i\}_{i=1}^{n_0}$ are the eigenvalues in decreasing order and corresponding normalized eigenvectors of $\E[\xbf \xbf^{\top}]$ respectively. By the spectral theorem, $\{\vbf_i\}_{i=1}^{n_0}$ is an orthonormal basis of $\R^{n_0}$.

    Suppose now that $d = 1$. Write $\Ubf$ as $\ubf \in \R^{n_0}$. Then,
    \begin{equation}
        \ubf^{\top} \E[\xbf \xbf^{\top}] \ubf = \sum_{i=1}^{n_0} \lambda_1 \ubf^{\top} \vbf_i \vbf_i^{\top} \ubf = \sum_{i=1}^{n_0} \lambda_i \ev{\vbf_i,\ubf}.
    \end{equation}
    Since $\{\vbf_i\}_{i=1}^{n_0}$ is an orthonormal basis of $\R^{n_0}$, we may write $\ubf = a_1\vbf_1 + \dots + a_{n_0}\vbf_{n_0}$ for some $\abf \in \R^{n_0}$ with unit norm. We have
    \begin{equation}
        \sum_{i=1}^{n_0} \lambda_i \ev{\vbf_i,\ubf} = \sum_{i=1}^{n_0} \lambda_i \bigg \langle \vbf_i, \sum_{j=1}^{n_0} a_j \vbf_j\bigg \rangle = \sum_{j=1}^{n_0} \lambda_i a_i.
    \end{equation}
    It is clear that this expression is maximized when $a_1 = 1$ and all other entries are zero. That is, $\ubf = \vbf_1$.

    Now assume $d > 1$ and suppose that we have found that $\ubf_1,\dots,\ubf_{d-1}$ correspond to $\vbf_1,\dots,\vbf_{d-1}$. Then, we must select $\ubf_d \in (\mathrm{span}\{\vbf_1,\dots,\vbf_{d-1}\})^{\perp}$. That is, $\ubf_d = a_d \vbf_d + a_{d+1} \vbf_{d+1} + \dots + a_{n_0} \vbf_{n_0}$. Following the same argument as in the case $d = 1$, we see that $\ubf_d = \vbf_d$ is optimal.
\end{proof}

Hence, the optimal rank $d$ projection is the one that maps data to the subspace of $\R^{n_0}$ spanned by the top $d$ eigenvectors of the covariance matrix $\E[\xbf \xbf^{\top}]$. As a result, we may obtain a $d$-dimensional representation of data in $\Xcal$ by projecting it onto this subspace and then writing the resulting vector in the basis $\{\vbf_1,\dots,\vbf_d\}$.

As was mentioned at the beginning of this section, PCA is also interpreted as finding the $d$-dimensional subspace of $\Xcal$ which maximally preserves the variance in the data. For this to be valid, assume that $p_{\xbf}$ has mean zero. Recalling (\ref{eq:variance_maximization}), we see that the expression
\begin{equation}
    \sum_{j=1}^d \E\big[\ev{\ubf_j,\xbf}^2\big]
\end{equation}
is the sum of the variance of the random vector $\xbf$ in the directions $\ubf_1,\dots,\ubf_d$.

In practice, the data distribution $p_{\xbf}$ and its covariance matrix are unknown. Instead, we have access to a dataset $\Dcal_{\xbf} = \{\xbf_i\}_{i=1}^N$. In this case, we replace the expectation with the sample mean. Our original optimization problem (\ref{eq:original_pca_problem}) becomes
\begin{equation}
    \underset{\Pbf \in \Pcal_d^{n_0}} \argmin \,\, \frac{1}{N} \sum_{i=1}^N ||\xbf_i-\Pbf \xbf_i||_2^2.
\end{equation}

Assuming zero mean and substituting the expectation with the sample mean in all derivations shows that the optimal rank $k$ subspace is spanned by the top $k$ eigenvectors of the (biased) sample covariance matrix
\begin{equation}
    \frac{1}{N} \sum_{i=1}^N \xbf \xbf^{\top}.
\end{equation}
Naturally, error is induced by using the sample covariance matrix instead of the true covariance matrix. This is beyond the scope of this thesis, but is addressed in \cite{wainwright2019high}, among others.

\textbf{Multidimensional Scaling (MDS).} MDS is a linear dimensionality reduction method that relies solely on a notion of dissimilarity between inputs. Depending on the context, this dissimilarity may or may not be a distance. It is the first method we introduce that explicitly factors a matrix of pairwise similarities. The idea of factoring a Gram matrix will henceforth be a recurring theme. 

There are several types of MDS depending on what assumptions are made about the notion of dissimilarity \cite{wang2012geometric}. In classical MDS, the dissimilarity is the Euclidean distance. Metric MDS handles the more general case where the dissimilarity is a metric. Finally, as the name suggests, non-metric MDS does not require the dissimilarity to be a metric. We follow the introduction of MDS in \cite{cox2000multidimensional} and restrict our attention to the classical case. 

Let $\Dbf \in \R^{N \times N}$ be the matrix of pairwise Euclidean distances associated with $\Dcal_{\xbf}$ so that $D_{ij} = \sqrt{(\xbf_i-\xbf_j)^{\top}(\xbf_i-\xbf_j)}$. Let $\Sbf$ be the entry-wise square of $\Dbf$, i.e., $S_{ij} = (D_{ij})^2$. MDS only assumes that we have access to the matrix $\Dbf$ (we need not even have access to $\Dcal_{\xbf}$). 

The high-level idea is to use $\Dbf$ to form the Gram matrix $\Gbf$ of inner products between the elements of $\Dcal_{\xbf}$. That is, $\Gbf = \Xbf \Xbf^{\top}$, where $\Xbf \in \R^{N \times d}$ is the matrix with $i$th row corresponding to $\xbf_i$. $\Gbf$ can be directly computed if we have access to $\Dcal_{\xbf}$, but the relationship between the dot product and the Euclidean distance also enables $\Gbf$ to be computed from $\Dbf$ alone. Either way, once $\Gbf$ is obtained, we seek a low-rank approximation of $\Gbf$ of the form $\Gbf_d = \Phibf_d \Phibf_d^{\top}$, where $\Phibf_d \in \R^{N \times d}$. The rows of $\Phibf_d$ provide a $d$-dimensional representation of the data that approximately preserves the inner product (and Euclidean distance).
 
Assume that the dataset is centred, i.e,
\begin{equation}
    \frac{1}{N} \sum_{i=1}^N x_{i,r} = 0 \hs\hs \forall \, r \in [n_0].
\end{equation}
Since we are employing Euclidean distance as our notion of dissimilarity, we have
\begin{equation}\label{eq:mds_inner_1}
    S_{ij} = \xbf_i^{\top} \xbf_i + \xbf_j^{\top} \xbf_j - 2\xbf_i^{\top} \xbf_j,
\end{equation}
which implies
\begin{align}\label{eq:mds_inner_2}
    \frac{1}{N} \sum_{i=1}^N S_{ij} = \frac{1}{N} \sum_{i=1}^N \xbf_i^{\top} \xbf_i + \xbf_j^{\top} \xbf_j \\
    \frac{1}{N} \sum_{j=1}^N S_{ij} = \xbf_i^{\top} \xbf_i + \frac{1}{N} \sum_{j=1}^N \xbf_j^{\top} \xbf_j.
\end{align}
We manipulate the last equation to obtain
\begin{equation}\label{eq:mds_inner_3}
    \frac{1}{N^2} \sum_{i=1}^N \sum_{j=1}^N S_{ij} = \frac{1}{N}\bigg(\sum_{i=1}^N \xbf_i^{\top} \xbf_i + \frac{1}{N} \sum_{i=1}^N \sum_{j=1}^N \xbf_j^{\top} \xbf_j\bigg) = \frac{2}{N}\sum_{i=1}^N \xbf_i^{\top} \xbf_i.
\end{equation}
Then, using (\ref{eq:mds_inner_1}), (\ref{eq:mds_inner_2}), and (\ref{eq:mds_inner_3}), we can obtain the entries of the Gram matrix $\Gbf$ from the entries of the square-distance matrix $\Sbf$ as follows:
\begin{align*}
    G_{ij} &= \xbf_i^{\top} \xbf_j \\
    &= -\frac{1}{2}\big(S_{ij} - \xbf_i^{\top} \xbf_i - \xbf_j^{\top} \xbf_j\big) \\
    &= -\frac{1}{2}\bigg(S_{ij} - \frac{1}{N}\sum_{j=1}^n S_{ij} - \frac{1}{N}\sum_{i=1}^N S_{ij} + \frac{1}{N^2} \sum_{i=1}^N \sum_{j=1}^N S_{ij} \bigg).
\end{align*}
Now, since $\Gbf = \Xbf \Xbf^{\top}$, where $\Xbf \in \R^{N \times n_0}$, $\Gbf$ is symmetric and positive semidefinite. It has the spectral decomposition $\Gbf = \Vbf \Lambdabf \Vbf^{\top}$. The Eckart-Young-Mirsky Theorem \cite{eckart1936approximation} states that
\begin{equation}\label{eq:mds_objective}
    \underset{\substack{\Gbf_d \in \R^{N \times N}\\ \mathrm{rank}(\Gbf_d)=d}}{\argmin} ||\Gbf-\Gbf_d||_F^2 = \Vbf_d \Lambdabf_d \Vbf_d^{\top},
\end{equation}
where $\Lambdabf_d = \diag(\lambda_1,\dots,\lambda_d)$ for $\lambda_1 \geq \dots \geq \lambda_d$ the top $k$ eigenvalues of $\Gbf$, and $\Vbf_d \in \R^{N \times d}$ is the matrix with the corresponding top $k$ eigenvectors as columns.

We obtain $d$-dimensional embeddings via the rows of
\begin{equation}\label{eq:mds_embeddings}
    \Phibf_d = \Vbf_d \Lambdabf_d^{\frac{1}{2}}.
\end{equation}
By (\ref{eq:mds_objective}), the dot products between these embeddings are the best approximation to the similarities encoded by the Gram matrix $\Gbf$.

While MDS is theoretically appealing, it has a key limitation: it cannot be applied to new data. If presented with an input $\xbf \in \Xcal \setminus \Dcal_{\xbf}$, the only means of embedding it is to re-run the MDS algorithm. On the other hand, PCA can handle this by projecting each new input onto the subspace spanned by the principal components.

\section{Manifold Learning}\label{section:manifold_learning}

The methods discussed in the previous section are of rather limited scope. Both PCA and MDS assume linear structure in the data, which may not be true for general machine learning problems. Specifically, PCA assumes that there exists a low-dimensional linear subspace (and hence an orthogonal projection from $\Xcal$ onto that subspace) which captures the latent structure of the data. MDS assumes that the Euclidean distances in the input space $\Xcal$ are meaningful and must be preserved.

A counterexample to the above assumptions is the "Swiss Roll dataset" introduced in \cite{tenenbaum2000global} and depicted in Figure \ref{fig:swiss_roll}. In this example, the data resides not in a linear subspace but rather on a nonlinear manifold structure. It is clear that the geodesic distance between two points on the data manifold does not correspond to Euclidean distance.

\begin{figure}
    \centering
    \includegraphics[width=0.5\linewidth]{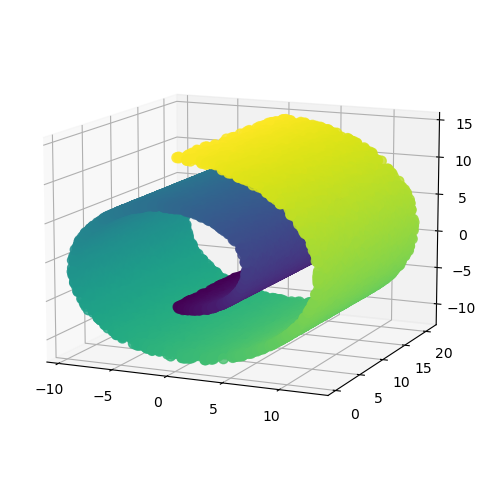}
    \caption{The Swiss Roll dataset from \cite{tenenbaum2000global}. This figure was generated using publicly available code from scikit-learn.\protect\footnotemark}
    \label{fig:swiss_roll}
\end{figure}

\textbf{ISOMAP.} The authors of \cite{tenenbaum2000global} adopt the manifold hypothesis introduced in Section \ref{section:paradigms}. Their ISOMAP algorithm seeks to produce low-dimensional representations that preserve the geodesic distance between inputs as opposed to the Euclidean distance. It does so in two steps: (1) approximating the geodesic distance for every pair of inputs, (2) applying MDS to the approximate geodesic distances.

In Step (1), we assume that the geodesic distance is locally well-approximated by the Euclidean distance. The idea is to approximate a geodesic between two points by a piecewise linear curve, where each line has length at most $\eps$. This is done by constructing a graph $G$, with vertex set equal to the set of distinct inputs in $\Dcal_{\xbf}$. \footnotetext{\url{https://scikit-learn.org/stable/auto_examples/manifold/plot_swissroll.html}} Two vertices $\xbf_i$ and $\xbf_j$ are connected by an edge of weight $||\xbf_i-\xbf_j||_2$ if and only if $||\xbf_i - \xbf_j||_2 < \eps$ for some fixed $\eps > 0$.\footnote{One can also instead choose to connect each vertex to its $K$ nearest neighbours (i.e., the $K$ closest inputs in terms of Euclidean distance).} Then, the approximate geodesic distance between any two points $\xbf_k$ and $\xbf_l$ is taken to be the length of the shortest path between them in $G$. This can be found by Dijkstra's algorithm \cite{dijkstra1959note}, for example.

There is theoretical justification for Step (1) for a certain class of manifolds. We present the main result in \cite{bernstein2000graph} for completeness. Assume throughout that $\Mcal$ is a compact submanifold of $\R^{n_0}$ which is isometric to a convex domain in $\R^d$. Let $d_{\Mcal}$ denote the geodesic distance, $V$ denote the volume of $\Mcal$, and $\eta_d$ denote the volume of the unit ball in $\R^d$.

\begin{definition}[Minimum Radius of Curvature]
    The minimum radius of curvature $r_0$ of $\Mcal$ is defined by
    \begin{equation}
        r_0 = \frac{1}{\max_{\gamma,t}||\gamma''(t)||_2},
    \end{equation}
    where $\gamma$ varies over geodesics in $\Mcal$ with $||g'||_2 = 1$ and $t$ is in the domain of $\gamma$.
\end{definition}

\begin{definition}[Minimum Branch Separation]
    The minimum branch separation of $\Mcal$ is the largest real number $s_0$ such that for all $\xbf, \ybf \in \Mcal$, if $||\xbf-\ybf||_2 < s_0$ then $d_{\Mcal}(x,y) \leq \pi r_0$.
\end{definition}

\begin{theorem}[Theoretical Guarantee for ISOMAP, \cite{bernstein2000graph}]
    Let $\lambda_1,\lambda_2, \mu \in (0,1)$, and $\eps > 0$ such that $\eps < s_0$ and $\eps \leq (2/\pi)r_0 \sqrt{24\lambda_1}$. Let $\Dcal_{\xbf} = \{\xbf_i\}_{i=1}^N$ be sampled i.i.d.\@ from a distribution with density $\alpha$ over $\Mcal$ with minimum value $\alpha_{\min}$. Suppose we construct a graph $G$ on $\Dcal_{\xbf}$ using Step (1) with parameter $\eps$. Let $d_G$ denote the shortest path distance in the graph. Suppose also that
    \begin{equation}
        \alpha_{\min} > \frac{1}{\eta_d \big(\lambda_2 \eps / 8\big)^d} \log\bigg(\frac{V}{\mu \eta_d (\lambda_2 \eps/16)^d}\bigg).
    \end{equation}
    Then, with probability at least $1-\mu$ and neglecting boundary effects,
    \begin{equation}
        (1-\lambda_1)d_{\Mcal}(x,y) \leq d_G(x,y) \leq (1+\lambda_2)d_{\Mcal}(x,y) \hs \forall \, x,y \in \Mcal.
    \end{equation}
\end{theorem}

The theoretical guarantee of ISOMAP's ability to capture global manifold structure is a major strength of the algorithm. Moreover, ISOMAP has been successfully employed in the healthcare, epidemiology, and natural language processing, among others \cite{ayesha2020overview}. However, ISOMAP also suffers from key weaknesses that can inhibit its success. It is particularly sensitive to the choice of neighbourhood radius $\eps$. A larger value of $\eps$ results in a worse approximation to the manifold structure. Shortest path distances in the graph can be much smaller than the true geodesic distance. This is referred to as "short-circuiting" \cite{balasubramanian2002isomap,van2007dimensionality}. Conversely, choosing $\eps$ too small risks disconnecting the graph, making approximation of all pairwise geodesic distances impossible \cite{van2008visualizing}. Finally — and most importantly to this thesis — ISOMAP is a non-parametric method that only produces features for inputs in $\Dcal_{\xbf}$. Like MDS, it is unable to handle new data in $\Xcal \setminus \Dcal_{\xbf}$.

\textbf{Locally Linear Embeddings (LLE).} Like ISOMAP, the LLE algorithm \cite{roweis2000nonlinear} is a nonlinear dimensionality reduction method inspired by the manifold hypothesis. As above, suppose that the data $\Dcal_{\xbf} = \{\xbf_i\}_{i=1}^N$ lies on a manifold $\Mcal$ embedded in $\R^{n_0}$. Locally, the manifold is approximately linear. LLE takes inspiration from this to approximately express each point $\Dcal_{\xbf}$ as a linear combination of its $K$ nearest neighbours through the following minimization:
\begin{equation}\label{eq:lle_input_space_lincomb}
    \min_{\Wbf \in \R^{N \times N}} \sum_{i=1}^N \bigg|\bigg|\xbf_i - \sum_{j=1}^N W_{ij}\xbf_j \bigg|\bigg|_2^2,
\end{equation}
subject to the constraints that $W_{ij} = 0$ if $\xbf_j$ is not one of the $K$ nearest neighbours to $\xbf_i$ and $\sum_{j=1}^N W_{ij} = 1$ for all $i \in [N]$. Note that the optimal choice of weights is invariant under rotation, re-scaling and translation of the data. This minimization problem may be solved in closed form with the method of Lagrange multipliers \cite{roweis2000nonlinear}.

After finding the optimal $\Wbf$, low-dimensional embeddings $\{\vbf_i\}_{i=1}^N \subseteq \R^d$ are produced by enforcing that the relationship (\ref{eq:lle_input_space_lincomb}) holds in the embedding space with the same weights. That is, a second minimization problem is solved:
\begin{equation}\label{eq:lle_emb_space_lincomb}
    \min_{\vbf_1,\dots,\vbf_N \in \R^d} \sum_{i=1}^N \bigg|\bigg|\vbf_i - \sum_{j=1}^N W_{ij}\vbf_j \bigg|\bigg|_2^2,
\end{equation}
subject to a zero mean constraint $\sum_{i=1}^N \vbf_i = \mathbf{0}$ and an identity covariance constraint $\frac{1}{N}\sum_{i=1}^N \vbf_i \vbf_i^{\top} = \Ibf_d$. The zero mean constraint captures the translation-invariance of the objective (\ref{eq:lle_emb_space_lincomb}). The identity covariance matrix prevents a degenerate solution. In other words, it ensures $\mathrm{span}\{\vbf_1,\dots,\vbf_N\} = \R^d$.

To solve this minimization problem, consider the expansion
\begin{align*}
    \sum_{i=1}^N \bigg|\bigg|\vbf_i - \sum_{j=1}^N W_{ij} \vbf_j \bigg|\bigg|_2^2 &= \sum_{i=1}^N \bigg[||\vbf_i||_2^2 - 2 \sum_{j=1}^N W_{ij} \vbf_i^{\top} \vbf_j + \sum_{j=1}^N \sum_{k=1}^N W_{ij} W_{ik} \vbf_j^{\top} \vbf_k \bigg] \\
    &= \sum_{i=1}^N \sum_{j=1}^N \vbf_i^{\top} \vbf_j \bigg(\delta_{ij} - 2W_{ij} + \sum_{k=1}^N W_{ki}W_{kj}\bigg) \\
    &= \sum_{i=1}^N \sum_{j=1}^N \vbf_i^{\top} \vbf_j \bigg(\delta_{ij} - W_{ij} - W_{ji} + \sum_{k=1}^N W_{ki} W_{kj}\bigg) \\
    &= \sum_{i=1}^N \sum_{j=1}^N M_{ij} (\vbf_i^{\top} \vbf_j) \\
    &= \tr(\Vbf^{\top} \Mbf \Vbf).
\end{align*}
where $\Vbf \in \R^{N \times d}$ is the matrix with $\vbf_1,\dots,\vbf_N$ as rows and $\Mbf \in \R^{N \times N}$ is the symmetric positive semidefinite matrix defined by
\begin{equation}
    M_{ij} = \delta_{ij} - W_{ij} - W_{ji} + \sum_{k=1}^N W_{ki} W_{kj}.
\end{equation}
Hence, our minimization problem becomes
\begin{equation}\label{eq:lle_trace_minimization}
    \min_{\Vbf \in \R^{N \times d}} \tr(\Vbf^{\top} \Mbf \Vbf)
\end{equation}
subject to the constraints $\Vbf^{\top} \onebf = \mathbf{0}$ and $\frac{1}{N} \Vbf^{\top} \Vbf = \Ibf_d$. Here, $\onebf$ denotes the vector of ones in $\R^d$. By the cyclic property of the trace and the fact that $\Mbf$ is PSD, we have
\begin{equation}
    \tr(\Vbf^{\top} \Mbf \Vbf) = \tr(\Mbf \Vbf \Vbf^{\top}) = \tr\big((\Mbf^{1/2})^{\top} \Vbf \Vbf^{\top} \Mbf^{1/2}\big).
\end{equation}
Moreover,
\begin{align*}
    \tr\big((\Mbf^{1/2})^{\top} \Vbf \Vbf^{\top} \Mbf^{1/2}\big) &= \sum_{i=1}^N \big((\Mbf^{1/2})^{\top} \Vbf \Vbf^{\top} \Mbf^{1/2}\big)_{ii} \\
    &= \sum_{i=1}^N \sum_{j=1}^N \sum_{k=1}^N M_{ji}^{1/2} (VV^{\top})_{jk} M_{ki}^{1/2} \\
    &= \sum_{i=1}^N \mbf_i \Vbf \Vbf^{\top} \mbf_i \\
    &= \sum_{i=1}^N ||\Vbf^{\top} \mbf_i||_2^2,
\end{align*}
where $\mbf_i$ is the $i$th column of $\Mbf^{1/2}$.

Let $\lambda_0 \leq \lambda_1 \leq \dots \leq \lambda_{N-1}$ be the eigenvalues of $\Mbf$ and let $\ubf_0,\ubf_1,\dots,\ubf_{N-1}$ be the associated normalized eigenvectors. The same argument as in our discussion of PCA in Section \ref{section:linear_dim_reduction} shows that the optimal columns of $\Vbf$ for the unconstrained (\ref{eq:lle_trace_minimization}) correspond to the eigenvectors $\ubf_0,\dots,\ubf_{d-1}$. However, $\lambda_0 = 0$ since $\onebf$ is an eigenvector of $\Mbf$ associated with the eigenvalue 0. Indeed, this is because the row sums of $\Mbf$ are zero:
\begin{equation}
    \sum_{j=1}^N M_{ij} = \sum_{j=1}^N \bigg(\delta_{ij} - W_{ij} - W_{ji} + \sum_{k=1}^N W_{ki} W_{kj}\bigg) = -1 + \sum_{k=1}^N W_{ki} \sum_{j=1}^N W_{kj} = 0.
\end{equation}

As a consequence, taking $\onebf$ to be a column of $\Vbf$ violates the constraint $\Vbf^{\top} \onebf = 0$. Hence, we instead take the columns of $\Vbf$ to be $\ubf_1,\dots,\ubf_d$. Since $\Mbf$ is symmetric, each of these vectors is orthogonal to $\onebf$, ensuring that the constraint is satisfied. The constraint $\frac{1}{N} \Vbf^{\top} \Vbf = \Ibf_d$ is also satisfied due to orthogonality. Having found the optimal $\Vbf$, the LLE representations for the inputs in $\Dcal_{\xbf}$ are the rows of $\Vbf$. 

In solving the optimization problems (\ref{eq:lle_input_space_lincomb}) and (\ref{eq:lle_emb_space_lincomb}), LLE produces representations which best preserve local neighbourhood structure rather than seeking to capture global manifold structure. Unlike ISOMAP — which involves the expensive step of running Dijkstra's algorithm for every point — LLE does not explicitly seek to approximate all pairwise geodesic distances, nor is there a theoretical guarantee that it does so. This reduces the effect of the aforementioned short-circuiting problem on LLE. Instead, a major shortcoming of LLE is that the identity covariance constraint does not prevent collapse to a near-trivial solution where the majority of inputs in $\Dcal_{\xbf}$ are clustered together in the embedding space and a small minority are distant outliers \cite{van2007dimensionality}.

Just as with MDS and ISOMAP, LLE cannot produce embeddings for inputs outside of $\Dcal_{\xbf}$. To compensate for this, its authors suggests training a neural network to approximate the LLE features on $\Dcal_{\xbf}$. Since the neural network is parametric, it can then produce features for any input in $\Xcal$.

\textbf{Laplacian Eigenmaps.} The Laplacian Eigenmaps method \cite{belkin2003laplacian} produces representations by building a graph structure on the input data using neighbourhood relationships as in ISOMAP and solves a similar optimization problem to (\ref{eq:lle_trace_minimization}) in LLE. 

Once again, suppose $\Dcal_{\xbf}$ is sampled from a manifold $\Mcal$ embedded in $\R^{n_0}$. We define the same graph $G$ as in ISOMAP, but re-weight the edges via the Gaussian kernel with hyperparameter $t > 0$:
\begin{equation}
    W_{ij} = \exp\bigg(-\frac{||\xbf_i-\xbf_j||_2^2}{t}\bigg).
\end{equation}
Let $\Dbf \in \R^{N \times N}$ be the diagonal matrix of row sums (equivalently, column sums) of $\Wbf$:
\begin{equation}
    \Dbf := \diag\bigg(\sum_{j=1}^N W_{j1}, \dots, \sum_{j=1}^N W_{jN}\bigg).
\end{equation}
Then, we define the \textit{graph Laplacian} $\Lbf := \Dbf - \Wbf$. The spectrum of this matrix will arise in the context of the minimization problem
\begin{equation}\label{eq:laplacian_emaps_minimization_problem}
    \min_{\vbf_1,\dots,\vbf_N \in \R^d} \sum_{i=1}^N \sum_{j=1}^N W_{ij} ||\vbf_i - \vbf_j||_2^2.
\end{equation}
This objective captures the intuition that nearby data points in the input space $\R^{n_0}$ should be mapped to nearby points in the embedding space $\R^d$. Manipulation of this objective yields
\begin{align*}
    \sum_{i=1}^N \sum_{j=1}^N W_{ij} ||\vbf_i - \vbf_j||_2^2 &= \sum_{i=1}^N \sum_{j=1}^N W_{ij} \big(||\vbf_i||_2^2 + ||\vbf_j||_2^2 - 2\vbf_i^{\top}\vbf_j\big) \\
    &= \sum_{i=1}^N D_{ii} ||\vbf_i||_2^2 + \sum_{j=1}^N D_{jj} ||\vbf_j||_2^2 - 2 \sum_{i=1}^N \sum_{j=1}^N W_{ij} \vbf_i^{\top} \vbf_j \\
    &= \sum_{i=1}^N \sum_{j=1}^N L_{ij} \vbf_i^{\top} \vbf_j \\
    &= \tr(\Vbf^{\top} \Lbf \Vbf),
\end{align*}
where $\Vbf \in \R^{N \times d}$ is the matrix with $\vbf_1,\dots,\vbf_N$ as rows. From the second-to-last line, we see that $\Lbf$ is PSD. We are in a similar setting to the LLE minimization problem (\ref{eq:lle_trace_minimization}) but with the difference that the rows of $\Wbf$ need not sum to 1. Put another way, the vertices in the graph may not have the same degree. To account for this in the optimization, we modify the translation-invariance and full-rank constraints from LLE. We impose $\Vbf^{\top} \Dbf \onebf = \mathbf{0}$ and $\Vbf^{\top} \Dbf \Vbf = \Ibf_d$. Hence, the optimization problem becomes
\begin{equation}\label{eq:laplacian_opt_problem}
    \underset{\substack{\Vbf^{\top}\Dbf \onebf = \mathbf{0} \\ \Vbf^{\top}\Dbf \Vbf = \Ibf_d}}{\argmin} \tr(\Vbf^{\top} \Lbf \Vbf).
\end{equation}
The optimal $\Vbf$ then arises from the generalized eigenvalue problem
\begin{equation}
    \Lbf \ubf = \lambda \Dbf \ubf.
\end{equation}
As in LLE, this has a zero eigenvalue $\lambda_0$ since $\ubf = \onebf$ is a solution. Hence, we take the eigenvectors $\ubf_1,\dots,\ubf_d$ associated with the next smallest eigenvalues $\lambda_1,\dots,\lambda_d$ to be the columns of $\Vbf$.

The Laplacian Eigenmaps approach on discrete datasets is a special case of a more general optimization problem involving the Laplace-Beltrami operator on the manifold $\Mcal$. In the discrete setting, we sought embeddings which preserve the graph structure on inputs. In the general case, we seek a function $\phibf: \Mcal \to \R^d$ that preserves local neighbourhoods. This is achieved by minimizing a continuous generalization of (\ref{eq:laplacian_emaps_minimization_problem}):
\begin{equation}
    \min_{||f||_{L^2(\Mcal; \mu)} = 1} \int_{\Mcal} ||\nabla f(\xbf)||^2 \, d\mu(\xbf),
\end{equation}
where $\mu$ is the standard measure on the Riemannian manifold $\Mcal$. The solution to this problem is to take $d$ eigenfunctions associated with the smallest nonzero eigenvalues of the Laplace-Beltrami operator, $L(f) := - \mathrm{div} \nabla f$. See \cite{belkin2003laplacian} for details.

Laplacian Eigenmaps is similar to LLE in that it prioritizes the preservation of distances for neighbouring points. The resulting optimization problem (\ref{eq:laplacian_opt_problem}) is akin to (\ref{eq:lle_trace_minimization}). As a result, like LLE, the Laplacian Eigenmaps method avoids the cost associated with running Dijkstra's algorithm, but is susceptible to collapse to a near-trivial solution \cite{van2007dimensionality,van2008visualizing}. We also note that, like the other approaches in this section, the discrete Laplacian Eigenmaps method does not produce features for inputs in $\Xcal \setminus \Dcal_{\xbf}$.

\section{Kernel Approximation}\label{section:kernel_approx}

In the previous two sections, we have introduced methods that aim to produce representations which map nearby inputs to nearby features, as measured by a notion of (dis)similarity or distance. Recall from Chapter \ref{chapter:rkhs} that positive semidefinite kernels provide a formalism for similarity between inputs.

In Section \ref{section:linear_dim_reduction}, we saw that classical MDS recovers features from the dot product kernel $\xbf \cdot \zbf$ defined on the input space $\Xcal \subseteq \R^{n_0}$ through a spectral decomposition of its associated Gram matrix on $\Dcal_{\xbf}$. In fact, features can be recovered from any PSD kernel $K$ in this way (consider (\ref{eq:mds_objective}) and (\ref{eq:mds_embeddings}) and replace $\Gbf$ by the Gram matrix of $K$). However, neither MDS nor the above manifold learning methods produce features for inputs in $\Xcal \setminus \Dcal_{\xbf}$. Additionally, computing the spectral decomposition of an $N \times N$ matrix has time complexity $O(N^3)$ \cite{trefethen2022numerical}. 

In this section, we discuss methods that go beyond Gram matrices to approximate PSD kernels on their entire domain $\Xcal$ rather than just the inputs in $\Dcal_{\xbf}$.

\textbf{Nyström Method.} Recall Mercer's Theorem (Theorem \ref{thm:mercer}), which states that any PSD kernel $K: \Xcal \times \Xcal \to \R$ satisfying Mercer's condition (\ref{eq:mercer_condition}) can be expressed as
\begin{equation}
    K(x,z) = \sum_{j=1}^{\infty} \lambda_j \phi_j(x) \phi_j(z),
\end{equation}
where $\{\lambda_j\}_{j=1}^{\infty}$ and $\{\phi_j\}_{j=1}^{\infty}$ are the eigenvalues and (orthonormal) eigenfunctions of the integral operator
\begin{equation}
    T_K(f) = \int_{\Xcal} K(x,z)f(z) \, p_x(dz).
\end{equation}
A natural choice of embedding map $\phibf: \Xcal \to \R^d$ is then the concatenation of eigenfunctions $\phi_1,\dots,\phi_d$ associated with the $d$ largest eigenvalues $\lambda_1,\dots,\lambda_d$ of $T_K$. That is, for $x \in \Xcal$, we have
\begin{equation}
    \phibf(x) = \big[\phi_1(x),\dots,\phi_d(x)\big].
\end{equation}
The Nyström method, introduced in the context of kernel methods by \cite{williams2000using}, provides a means of approximating the Mercer eigenfunctions. It approximates the eigenvalue problem for $T_K$ with Monte Carlo integration. Let $M \in \N$ and draw $M$ samples $x_1,\dots,x_M$ i.i.d.\@ from $p$ (or simply take $M$ samples from $\Dcal_x$). Then, we approximate the equation $T_K \phi_i = \lambda_i \phi_i$ for $i \in [d]$ with
\begin{equation}\label{eq:nystrom_approximate_eigenproblem}
    \frac{1}{M} \sum_{k=1}^M K(x_k,z) \phi_i(x_k) \approx \lambda_i \phi_i(z) \hs\hs \forall \, z \in \Xcal.
\end{equation}
Furthermore, we approximate the orthogonality constraint
\begin{equation}
    \int_{\Xcal} \phi_i(x)\phi_j(x) \, p_x(dz) = \delta_{i,j}
\end{equation}
by
\begin{equation}
    \frac{1}{M} \sum_{k=1}^M \phi_i(x_k)\phi_j(x_k) \approx \delta_{i,j}.
\end{equation}
Define $\vbf_i = [\phi_i(x_1),\dots,\phi_i(x_M)]^{\top}$. From the above, $||\vbf_i||_2 \approx \sqrt{M}$. Let $\Gbf \in \R^{M \times M}$ be the Gram matrix of $K$ associated with $x_1,\dots,x_M$. Then, from (\ref{eq:nystrom_approximate_eigenproblem}), we have
\begin{equation}\label{eq:nystrom_gram_eigenproblem}
    \Gbf \vbf_i \approx (M\lambda_i)\vbf_i.
\end{equation}
Let $\lambda_1^{(\Gbf)} \geq \dots \geq \lambda_M^{(\Gbf)}$ and $\ubf_1^{(\Gbf)},\dots,\ubf_M^{(\Gbf)}$ be the eigenvalues and associated (orthonormal) eigenvectors of $\Gbf$. Then, from (\ref{eq:nystrom_gram_eigenproblem}), we have the \textit{Nyström approximations}
\begin{equation}
    \lambda_i \approx \frac{\lambda_i^{(\Gbf)}}{M} \hs\hs \vbf_i \approx \sqrt{M} \ubf_i.
\end{equation}
for all $i \in [d]$.

To generalize the approximation to inputs $z \in \Xcal$ which were not in the sample, we plug back into (\ref{eq:nystrom_approximate_eigenproblem}) to obtain
\begin{equation}
    \phi_i(z) \approx \frac{\sqrt{M}}{\lambda_i^{(\Gbf)}} \sum_{k=1}^M K(x_k,z) u_{i,k}^{(\Gbf)}.
\end{equation}
for all $i \in [d]$, where $u_{i,k}^{(\Gbf)}$ denotes the $k$th entry of the eigenvector $\ubf_i^{(\Gbf)}$.

Practically speaking, this means that we can approximately compute a Mercer feature map using $M$ data samples. We think of $M$ as being much smaller than the size $N$ of the available dataset $\Dcal_{\xbf}$. The eigendecomposition of $\Gbf$ has an $O(M^3)$ cost and the evaluation of eigenfunctions on new data has an $O(M)$ cost. 

Empirically, \cite{williams2000using} find that the Nyström method produces performant features for small-scale classification tasks. However, they do not provide a theoretical performance guarantee for the method. It is unclear how many samples $M$ are required to ensure low error (with high probability) and if any assumptions other than Mercer's condition must be satisfied. The authors of \cite{drineas2005nystrom} propose a similar method with a more sophisticated sampling strategy and provide a theoretical performance guarantee in the context of approximating the Gram matrix of a large dataset.

\textbf{Random Fourier Features.} The seminal work of \cite{rahimi2007random} proposes a method for approximating shift-invariant PSD kernels without the use of training data.

\begin{definition}[Shift-Invariant Kernel]
    A kernel $K: \Xcal \times \Xcal \to \R$ is \textbf{shift-invariant} if there exists a univariate function $g: \Xcal \to \R$ such that
    \begin{equation}
        K(x,z) = g(x-z) \hs\hs \forall \, x,z \in \Xcal.
    \end{equation}
    In other words, a shift-invariant kernel depends only on the difference between its arguments.
\end{definition}

We have the following key result from harmonic analysis concerning shift-invariant kernels.

\begin{theorem}[Bochner]
    A continuous shift-invariant kernel $K: \R^{n_0} \times \R^{n_0} \to \R$ is positive semidefinite if and only if it is the Fourier transform of a non-negative measure $\mu$ on $\R^{n_0}$, i.e.,
    \begin{equation}
        K(\xbf,\zbf) = g(\xbf-\zbf) = \int_{\R^{n_0}} e^{i\xibf^{\top}(\xbf-\zbf)} \, d\mu(\xibf) \hs\hs \forall \, \xbf,\zbf \in \R^{n_0}.
    \end{equation}
    Moreover, if $K$ is properly scaled, then the measure $\mu$ in the above is a probability measure.
\end{theorem}

Henceforth, suppose we are in the setting of Bochner's Theorem with an appropriately scaled kernel $K$ so that $\mu$ is a probability measure. We can recover $\mu$ by interpreting $K$ as a univariate function and computing its Fourier transform. Since $K$ is real-valued, we have 
\begin{equation}
    K(\xbf,\zbf) = \int_{\R^{n_0}} e^{i\xibf^{\top} (\xbf-\zbf)} \, d\mu(\xibf) = \int_{\R^{n_0}} \cos\big(\xibf^{\top}(\xbf-\zbf)\big) \, d\mu(\xibf) \hs\hs \forall \, \xbf,\zbf \in \R^{n_0}.
\end{equation}

Using the trigonometric sum identity, we have for all $\xbf,\zbf \in \R^{n_0}$,
\begin{equation}
    \cos\big(\xibf^{\top}(\xbf-\zbf)\big) = \cos(\xibf^{\top}\xbf)\cos(\xibf^{\top}\zbf) + \sin(\xibf^{\top}\xbf)\sin(\xibf^{\top}\zbf) = 
    \begin{bmatrix}
        \cos(\xibf^{\top} \xbf) \\
        \sin(\xibf^{\top} \xbf)
    \end{bmatrix}
    \cdot
    \begin{bmatrix}
        \cos(\xibf^{\top} \zbf) \\
        \sin(\xibf^{\top} \zbf)
    \end{bmatrix}.
\end{equation}

Thus, defining the two-dimensional representation $\phibf_{\xibf}(\xbf) = [\cos(\xibf^{\top} \xbf), \sin(\xibf^{\top} \xbf)]^{\top}$ for all $\xbf \in \R^{n_0}$, we have
\begin{equation}
    K(\xbf,\zbf) = \E_{\xibf \sim \mu}\big[\phibf_{\xibf}(\xbf) \cdot \phibf_{\xibf}(\zbf)] \hs\hs \forall \, \xbf,\zbf \in \R^n.
\end{equation}
In this way, we have obtained an unbiased estimator for $K$ at all pairs of points in the input space $\R^n$ using the dot product  between two-dimensional representations. To reduce the variance of this estimator, we draw $d$ samples $\xibf_1,\dots,\xibf_d$ i.i.d.\@ from $\mu$. Then, for all $\xbf \in \R^{n_0}$, we have the $2d$-dimensional representation
\begin{equation}\label{eq:rff}
    \phibf(\xbf) = \frac{1}{\sqrt{d}} \big[\cos(\xibf_1^{\top}\xbf),\dots,\cos(\xibf_d^{\top}\xbf),\sin(\xibf_1^{\top}\xbf),\dots,\sin(\xibf_d^{\top}\xbf)\big].
\end{equation}
Then, we have the dot product
\begin{equation}
    \phibf(\xbf)^{\top}\phibf(\zbf) = \frac{1}{d}\sum_{j=1}^d \phibf_{\xibf_j}(\xbf)\phibf_{\xibf_j}(\zbf).
\end{equation}

The authors of \cite{rahimi2007random} prove pointwise and uniform approximation guarantees for the representations (\ref{eq:rff}), which they call \textit{Random Fourier Features}. The key technical tool is Hoeffding's inequality, a standard concentration of measure result \cite{vershynin2018high}.

\begin{lemma}[Hoeffding's Inequality]
    Let $X_1,\dots,X_d$ be independent real-valued random variables on a probability space $(\Omega, \Fcal, \Prob)$. Assume that there exist $m_i, M_i \in \R$ such that $m_i \leq X_i \leq m_i$ for every $1 \leq i \leq d$. Then, for any $\eps > 0$,
    \begin{equation}
        \Prob\bigg(\bigg|\sum_{i=1}^d \big(X_i - \E[X_i]\big)\bigg| \geq \eps\bigg) \leq 2\exp\bigg(-\frac{2\eps^2}{\sum_{i=1}^d (M_i-m_i)^2}\bigg).
    \end{equation}
\end{lemma}

\begin{theorem}[Pointwise Kernel Approximation Guarantee, \cite{rahimi2007random}]
    Let $\xbf,\zbf \in \R^n$ and let $\phibf: \R^{n_0} \to \R^{2d}$ be the random Fourier feature map specified in (\ref{eq:rff}). Then, for any $\eps > 0$,
    \begin{equation}
        \Prob\big(|\phibf(\xbf)^{\top}\phibf(\zbf) - K(\xbf,\zbf)| \geq \eps\big) \leq 2\exp\bigg(-\frac{d\eps^2}{2}\bigg).
    \end{equation}
\end{theorem}
\begin{proof}
    This follows immediately from Hoeffding's inequality and the fact that $\phibf_{\xibf}(\xbf)^{\top}\phibf_{\xibf}(\zbf) \in [-1,1]$ for any $\xibf \in \R^n$.
\end{proof}

\begin{theorem}[Uniform Kernel Approximation Guarantee, \cite{rahimi2007random}]
    Let $\Mcal \subseteq \R^{n_0}$ be compact with diameter $\diam(\Mcal)$ and $\phibf: \R^{n_0} \to \R^{2d}$ be the random Fourier feature map specified in (\ref{eq:rff}). Then, for any $\eps > 0$,
    \begin{equation}
        \Prob\bigg(\sup_{\xbf,\zbf \in \Mcal} \big|\phibf(\xbf)^{\top} \phibf(\zbf) - K(\xbf,\zbf)\big| \geq \eps\bigg) \leq 2^8 \bigg(\frac{\sigma \diam(\Mcal)}{\eps}\bigg)^2 \exp\bigg(-\frac{d\eps^2}{4(n+2)}\bigg),
    \end{equation}
    where $\sigma^2 = \E_{\xibf \sim \mu}[\xibf^{\top} \xibf]$. Moreover, $\sup_{\xbf,\zbf \in \Mcal} \big|\phibf(\xbf)^{\top} \phibf(\zbf) - K(\xbf,\zbf)\big| \leq \eps$ with any constant probability if
    \begin{equation}
        d = \Omega\bigg(\frac{d}{\eps^2} \log \frac{\sigma \diam(\Mcal)}{\eps}\bigg).
    \end{equation}
\end{theorem}

Unlike MDS and the manifold learning methods in the previous section, the Nyström method and Random Fourier Features produce representations for every element of $\Xcal$. However, both require that the PSD kernel $K$ be explicitly defined for every pair of inputs. Random Fourier Features are even more restrictive by also requiring that $K$ be shift-invariant. In the next chapter, we will see how contrastive learning circumvents these assumptions.

%% file: Chapters/contrastive.tex
\chapter{Contrastive Learning}\label{chapter:contrastive}

The dimensionality reduction methods explored in the previous chapter fulfill the desideratum of producing low-dimensional features from high-dimensional data. However, these methods also have key limitations that motivated the development of contrastive learning. Each method from the previous chapter has at least one of the following three shortcomings: the inability to generalize to new data, restrictive assumptions, and computational expense. 

PCA and classical MDS assume that data is concentrated on a low-dimensional linear subspace of the high-dimensional input space. This strong assumption is violated by datasets such as the Swiss Roll in Figure \ref{fig:swiss_roll}. As discussed at the end of Section \ref{section:linear_dim_reduction}, PCA can generalize to new inputs by projecting them onto the subspace spanned by the principal components, but MDS cannot. Moreover, these algorithms compute the spectral decomposition of an $N \times N$ matrix, which is an $O(N^3)$ operation.

The manifold learning algorithms in Section \ref{section:manifold_learning} do away with this linearity assumption but still involve the expensive spectral decomposition step. Additionally, none of them can generalize to new data as they are non-parametric methods which require that the embeddings for each input in the training set be explicitly stored.

Finally, the kernel approximation algorithms in Section \ref{section:kernel_approx} do exhibit an ability to generalize to new data but require that a PSD kernel $K$ be explicitly specified over all pairs in $\Xcal \times \Xcal$. In the case of Random Fourier Features, $K$ must be a shift-invariant kernel. In the Nyström method, $K$ must be evaluated $M$ times (where $M$ is the number of samples used in the eigenfunction approximation) for every new input, resulting in slow inference.

Contrastive learning addresses the deficiencies highlighted above by learning features from data with neural networks. This leverages the fact that neural networks are parametric function approximators which have empirically exhibited a strong generalization ability. The only assumption made on the structure of the data is that similar inputs are "close" in latent space (this will be formalized in the coming sections). Furthermore, important algorithmic, software, and hardware advances have facilitated training and inference with neural networks on graphical processing units (GPUs) \cite{rumelhart1986learning,kingma2014auto,paszke2019pytorch}.

At its core, contrastive learning adopts the same idea as manifold learning in that it wishes to map similar points to nearby representations and dissimilar points far away from each other. Unlike in MDS or kernel approximation, the machine learning practitioner is not required to define a notion of similarity (or kernel) on every pair of points in the data. Instead, it suffices to devise a means to automatically sample pairs of similar inputs (referred to as \textit{positive samples}) and pairs of dissimilar inputs (referred to as \textit{negative samples}). The desired representations are then learned via the minimization of a loss function --- called the \textit{contrastive loss} --- with stochastic gradient descent. The loss relies on a user-defined notion of similarity in the feature space (e.g., dot product between representations). It encourages this similarity to be as high as possible for positive samples and as low as possible for negative samples.

In this chapter, we explore the development of contrastive learning and specific instantiations of it in NLP and computer vision. From there, we cover recent theoretical works which derive explicit connections between contrastive learning and earlier dimensionality reduction methods. This will illustrate that contrastive learning is, in essence, a parametric way of learning a low-rank factorization of a kernel. 

\section{Origins of Contrastive Learning}\label{section:origins}

The contrastive loss first appeared in \cite{chopra2005learning}, which sought to learn a notion of distance between images for the task of face verification. The latter is a supervised learning problem that involves verifying whether an image of a face is associated with the correct person. Let $\Xcal$ denote the input space of face images under consideration and let $\Ycal$ denote the set of individuals whose faces appear in $\Xcal$. We are given an image $x \in \Xcal$ and a claimed identity $z \in \Ycal$ as input. Let $y \in \Ycal$ denote the true identity (i.e., the label) of the person in the image $x$. The problem is to determine whether $z = y$.

Viewing this problem through the lens of classification, each distinct individual forms their own class. However, when $|\Ycal|$ is large (e.g., on the order of thousands), standard classification methods from Section \ref{section:classification} become impractical. Instead, given a new input $(x,z)$, the authors propose comparing the image $x$ to the images in the training set with label $z$. The problem is then to determine whether or not $x$ is "sufficiently similar" to these images. 

To do this, the authors learn a neural network feature map $\phibf_{\theta}: \Xcal \to \R^d$. The semantic distance between two images $x_1,x_2$ is then taken to be $d_{\theta}(x_1,x_2) = ||\phibf_{\theta}(x_1)-\phibf_{\theta}(x_2)||_2$. We desire $d_{\theta}(x_1,x_2)$ to be small for positive samples (i.e., two images of the same person) and to be large for negative samples (i.e, images of different people). The contrastive loss function from \cite{chopra2005learning} is then\footnote{We present a simplified form of the loss which does away with additional constants.}:
\begin{equation}\label{eq:contrastive_loss}
    \ell(\theta; x, y) = y\,d_{\theta}^2(x_1,x_2) + (1-y)\exp\big(- d_{\theta}(x_1,x_2)\big),
\end{equation}
where $y=1$ if $(x_1,x_2)$ is a positive sample and $y=0$ if $(x_1,x_2)$ is a negative sample.

This loss indeed encourages minimizing the distance for positive samples and maximizing it for negative samples. Up to constants, we have
\begin{equation}
    \ell(\theta;x,y) =
    \begin{cases}
        \frac{d_{\theta}^2}{2} & \text{if } y = 1 \\
        e^{-d_{\theta}} & \text{if } y = 0.
    \end{cases}
\end{equation}
Taking derivatives, we have
\begin{equation}\label{eq:contrastive_derivative}
    \frac{\partial \ell}{\partial d_{\theta}} = 
    \begin{cases}
        d_{\theta} & \text{if } y = 1 \\
        -e^{-d_{\theta}} & \text{if } y = 0
    \end{cases}
\end{equation}
and
\begin{equation}\label{eq:contrastive_hessian}
    \frac{\partial^2 \ell}{\partial d_{\theta}^2} = 
    \begin{cases}
        1 & \text{if } y = 1 \\
        e^{-d_{\theta}} & \text{if } y = 0.
    \end{cases}
\end{equation}
From (\ref{eq:contrastive_derivative}), we see that the loss is increasing in $d_{\theta}$ for positive samples and decreasing in $d_{\theta}$ for negative samples. From (\ref{eq:contrastive_hessian}), we have that the loss is convex in $d_{\theta}$.

The follow-up work \cite{hadsell2006dimensionality} further establishes the motivation for contrastive learning. The authors highlight the deficiencies in manifold learning, with particular emphasis on the inability to generalize to new samples and the need for a pre-specified notion of similarity (or dissimilarity) over all pairs of inputs. Their contrastive loss, shown below, is similar to \cite{chopra2005learning}:
\begin{equation}\label{eq:hadsell_loss}
    \ell(\theta; x_1,x_2,y) = y d_{\theta}^2(x_1,x_2) + (1-y)\big(\max\, \{0, m-d_{\theta}(x_1,x_2)\}\big)^2,
\end{equation}
where $m$ is a margin hyperparameter and $d_{\theta}$ is taken to be the $\ell_2$ distance between features (instead of the $\ell_1$ distance). The margin hyperparameter effectively acts as a soft constraint that encourages dissimilar inputs to be at a distance of at least $m$ from each other in feature space.

It should be noted that the approach in \cite{hadsell2006dimensionality} requires that all pairs of inputs $(x_1,x_2) \in \Xcal$ receive a binary label (i.e., "similar" or "dissimilar"). The loss over the full dataset is then the average over all pairs. In the supervised setting, where two inputs can be deemed similar if and only if they have the same label, this is reasonable. However, in the unsupervised setting, defining a notion of similarity over all pairs is nontrivial. In subsequent sections, we venture into self-supervised learning by exploring methods that define a strategy for sampling positive and negative pairs.

Additionally, we note that it is necessary to draw both positive and negative samples. Had the contrastive loss been defined only for positive samples, it would be trivially minimized by mapping all inputs to the same point in $\R^d$.

\section{Noise-Contrastive Estimation}\label{section:nce}

As we have seen, contrastive learning boils down to the problem of distinguishing between positive and negative samples. Put another way, we wish to determine whether a pair of inputs was drawn from the distribution of positive samples or the distribution of negative samples. Noise-contrastive estimation (NCE) \cite{gutmann2010noise,gutmann2012noise} is a statistical method for modelling an unknown distribution that makes use of this idea of distinguishing between two distributions. The major breakthroughs in contrastive learning for NLP \cite{mikolov2013distributed} and computer vision \cite{oord2018representation,chen2020simple} involved contrastive loss functions based on NCE. 

NCE was proposed as an alternative to maximum likelihood estimation (MLE), whose shortcomings we will briefly illustrate. Let $p_1$ denote the unknown distribution of interest over a data space $\Xcal$ (not necessarily finite). Let $p_{\theta}$, $\theta \in \Theta$, denote a parametric model which has the goal of approximating $p_1$. 

In MLE, $\theta$ is chosen so as to maximize the \textit{log-likelihood function}
\begin{equation}\label{eq:log_likelihood}
    \log \Lcal(\theta) = \sum_{i=1}^{N^{(1)}} \log p_{\theta}(x_i)
\end{equation}
for some dataset $\Dcal_1 = \{x_1,\dots,x_{N^{(1)}}\} \subseteq \Xcal$ sampled i.i.d.\@ from $p_1$. Assume there exists $\theta^* \in \Theta$ such that $p_1 = p_{\theta^*}$. A classic example of MLE is estimating the mean and variance of a Gaussian distribution given a dataset of samples from that distribution.

MLE has a key limitation: it requires normalized probability densities. If we have a model (e.g., a neural network) which outputs a value $g_{\theta}(x)$ for each $x \in \Xcal$, we need to compute the normalization factor\footnote{This is often referred to as the \textit{partition function} in the literature.} $\int_{\Xcal} g_{\theta}(x) \, dx$ to obtain a probability density. However, computing this integral over a large space $\Xcal$ after every update to $\theta$ would be prohibitively expensive.

NCE overcomes the above issue by moving to the binary classification setting that we introduced in Section \ref{section:classification}. We construct a supervised problem of distinguishing (contrasting) between our target distribution $p_1$ and a known noise distribution $p_0$ that is easy to sample from (e.g., the uniform distribution over a subset of $\Xcal$). We form a dataset $\Dcal$ by taking the $N^{(1)}$ positive samples from $\Dcal_1$ and drawing $N^{(0)}$ negative samples from $p_0$. Let $k = N^{(0)}/N^{(1)}$ and $N = N^{(0)} + N^{(1)}$.

We use the modified sigmoid activation function $\sigma_k$ and the binary cross-entropy classification loss, which we recall here:
\begin{equation}
    \ell(\theta; x,y) = -y \log \sigma_k\big(s_{\theta}(x)\big) - (1-y)\log \sigma_k\big(1-s_{\theta}(x)\big),
\end{equation}
where $y=1$ for positive samples and $y=0$ for negative samples.

Summing over all samples, we have the loss on $\Dcal$,
\begin{equation}\label{eq:nce_loss}
    L_{\Dcal}(\theta) = -\frac{1}{N^{(1)}} \sum_{i=1}^{N^{(1)}} \log \sigma_k\big(s_{\theta}(x_i)\big) - \frac{1}{N^{(0)}} \sum_{j=1}^{N^{(0)}} \log \sigma_k\big(1-s_{\theta}(x_j)\big).
\end{equation}

As discussed in Section \ref{section:classification}, this is a proper scoring rule. Hence, we are in the setting of Corollary \ref{cor:centred_pmi_limit}, i.e.,
\begin{equation}
    \lim_{N \to \infty} s_{\theta^*_N}(x) = \log\bigg(\frac{p_1(x)}{p_0(x)}\bigg).
\end{equation}
In the limit of infinite samples, we may approximate $p_1$ using the learned function $s_{\theta}$ and the known noise distribution $p_0$.

In the context of contrastive learning, we will take $p_1$ to be the distribution over positive samples and $p_0$ to be the distribution over negative samples. We interpret $x$ as a pair of inputs $x = (x_1,x_2)$. By definition of a proper scoring function and the Strong Law of Large Numbers, we have that, in the limit of infinite samples, $\sigma_k(s_{\theta}) = p(y=1|x)$. That is, $\sigma_k(s_{\theta})$ indicates how likely $(x_1,x_2)$ is to be a pair of similar inputs. 

\section{Word Embeddings}\label{section:word_embeddings}

Word embeddings constitute one of the first successful applications of NCE to representation learning. The word2vec algorithm \cite{mikolov2013distributed} --- also called skip-gram with negative sampling (SGNS) --- converts the pretext task of language modelling (i.e., predicting missing words in a span of text) to the contrastive problem of distinguishing between positive and negative samples. In this context, a positive sample is a pair of words which appear together in the data and a negative sample is a pair of independently sampled words. With a contrastive loss function similar to NCE, word representations which can subsequently be used as inputs to downstream task models are learned.

Let us formalize the language modelling task. We iterate over the words in a large collection of text (called a \textit{corpus}) $\mathcal{D}^{pre} = \{x_1,x_2,\dots,x_T\}$.\footnote{We understand $\Dcal^{pre}$ to be a multiset since the same word may appear multiple times in a corpus.} At each step, the task is to predict the surrounding words ($m$ words to the left, $m$ words to the right) given only the current (\textit{target}) word. Hence, the objective function is the log-likelihood of the true context words given the target word
\begin{equation}\label{eq:next_word_prediction}
    \frac{1}{T} \sum_{t=1}^T \sum_{-m \leq j \leq m, j \neq 0} \log p_{\theta}(x_{t+j}|x_t).
\end{equation}
For our model $p_{\theta}$, we make use of two word embedding maps, $\phibf: \Xcal \to \R^d$ and $\psibf: \Xcal \to \R^d$. We use $\phibf$ to represent the word being conditioned on and $\psibf$ to represent the context words we wish to predict. Since the vocabulary $\Xcal$ is finite, we can learn the embedding vectors directly. That is, the parameters of our model are the entries of the embeddings themselves. To measure the compatibility of a context word $z$ with a target word $x$, we compute the dot product $\phibf(x)^{\top}\psibf(z)$. To convert this to a probability, we perform the softmax operation
\begin{equation}\label{eq:softmax_vocab}
    p_{\theta}(z|x) = \frac{\exp\big(\phibf(x)^{\top} \psibf(z)\big)}{\sum_{z' \in \Xcal} \exp\big(\phibf(x)^{\top} \psibf(z')\big)}.
\end{equation}
This softmax is illustrative of the problem with MLE that we discussed in Section \ref{section:nce}. A sum over the vocabulary (which may contain hundreds of thousands or even millions of words) is not practical. NCE provides us with a more efficient alternative.

In this context, we define positive samples to be (target word, context word) pairs observed in the data. On the other hand, negative samples are taken to be (target word, random word) pairs. Hence, as we iterate over the words $x_t$ in the corpus, we take $(x_t,x_{t-m}),(x_t,x_{t-m+1}), \dots, (x_t,x_{t-1}), (x_t, x_{t+1}), \dots (x_t,x_{t+m})$ to be positive samples. For each positive sample, we produce $k$ negative samples 
$(x_t,z_1)$, $\dots$, $(x_t,z_k)$ by drawing $z_1,\dots,z_k$ i.i.d.\@ from the unigram distribution\footnote{The unigram distribution $p$ over $\Xcal$ is such that $p(w) = \frac{N_w}{T}$, where $N_w$ denotes the number of times the word $w$ appears in the corpus. In reality, SGNS employs a smoothed version of the unigram distribution so as to undersample the most frequent words (e.g., "the", "a", "of").}.

Hence, we are comparing the distribution over (context, target) pairs to the one over (random, target) pairs. We have $p_1(x,z) = \phat(x,z)$, which we define to be the (empirical) joint distribution over contexts and targets. That is,
\begin{equation}
    \phat(x,z) := \frac{N_{(x,z)}^{(1)}}{N},
\end{equation}
where $N_{(x,z)}^{(1)}$ is the number of times $(x,z)$ appears as as a positive pair in the corpus and $N = 2mT$ is the total number of positive pairs. Meanwhile, our noise distribution is $p_0(x,z) = \phat(x)\phat(z)$, the product of unigram (i.e., marginal) distributions. 

\begin{definition}[SGNS Loss]
    For a single positive sample $(x,z)$, the \textbf{SGNS loss} is a modified binary cross-entropy loss
    \begin{equation}\label{eq:w2v_one_sample}
        \ell(\phibf,\psibf;x,z,\{z_j\}_{j=1}^k) = -\log \sigma\big(\phibf(x)^{\top} \psibf(z)\big) - \sum_{j=1}^k \log\big(1-\sigma\big(\phibf(x)^{\top} \psibf(z_j)\big)\big).
    \end{equation}
\end{definition}
The loss over the full corpus is then
\begin{equation}\label{eq:w2v_corpus_loss}
    L_{\Dcal^{pre}}(\phibf,\psibf) = \frac{1}{N} \sum_{(x,z) \in \Xcal \times \Xcal} \bigg[-N_{(x,z)}^{(1)} \log \sigma\big(\phibf(x)^{\top} \psibf(z)\big) - N^{(0)}_{(x,z)} \log \big(1-\sigma\big(\phibf(x)^{\top}\psibf(z)\big)\big)\bigg],
\end{equation}
where $N^{(0)}_{(x,z)}$ denotes the number of times $(x,z)$ is drawn as a negative sample.
\begin{theorem}\label{thm:word2vec_kernel}
    The minimizer $(\phibf,\psibf)$ of $L_{\Dcal^{pre}}$ satisfies
    \begin{equation}\label{eq:word_emb_as_matrix_factorization}
        \phibf(x)^{\top} \psibf(z) = \log\bigg(\frac{\phat(x,z)}{\phat(x)\phat(z)}\bigg) - \log k \hs\hs \forall \, x, z \in \Xcal.
    \end{equation}
\end{theorem}
\begin{proof}
    It is immediate that $L_{\Dcal^{pre}}$ is a proper scoring function (the same proof can be followed as for the cross-entropy loss in Section \ref{section:classification}). Therefore, we can apply Theorem \ref{thm:shifted_pmi} to conclude (\ref{eq:word_emb_as_matrix_factorization}) at the optimum of the loss.
\end{proof}

\begin{remark}
    Had we used the activation function $\sigma_k$ instead of $\sigma$, we would have obtained $\phibf(x)^{\top} \psibf(z) = \log \frac{\phat(x,z)}{\phat(x)\phat(z)}$ by Corollary \ref{cor:centred_pmi}.
\end{remark}

The quantity $\log \frac{\phat(x,z)}{\phat(x)\phat(z)}$ is the (empirical) log-odds of sampling $(x,z)$ as a positive pair. Moreover, the authors of \cite{levy2014neural} note that it is the \textit{pointwise mutual information} (PMI) between the words $x$ and $z$. It is an information-theoretic quantity that measures the dependence between $x$ and $z$ in the following sense: if $x$ and $z$ appear together more frequently than would be expected if they were "independent", their PMI is positive; if they appear less frequently than expected, their PMI is negative. Equivalently, we can interpret PMI as a measure of how the presence of the word $x$ affects the probability of observing $z$.

Hence, SGNS embeddings implicitly approximate the kernel $\mathrm{PMI}(x,z)$ by the quantity $\phibf(x)^{\top}\psibf(z)$. We remark that the former is not PSD \cite{allen2019vec} and the latter is not symmetric in general. However, $\exp(\mathrm{PMI}(x,z))$ is a PSD kernel by Example \ref{ex:exp_pmi}. If we modified the algorithm to learn a single embedder $\phibf: \Xcal \to \R^d$ and used the score $\phibf(x)^{\top}\phibf(z)$ to measure compatibility, then $\exp\big(\phibf(x)^{\top}\phibf(z)\big)$ would be a PSD kernel by Proposition \ref{prop:exp_kernel}. In this way, we would be approximating a PSD kernel with another that is more straightforward to compute.

The reason why SGNS and other word embedding algorithms train two embedders $\phibf$ and $\psibf$ is linguistically motivated \cite{allen2019vec}. $\mathrm{PMI}$ (and thus $\phibf^{\top}\psibf$) captures word co-occurrences in the data, which is a proxy for \textit{word relatedness}. The latter concerns whether a pair of words fall under the same topic (e.g., sports, politics, science). Practitioners are ultimately concerned with learning \textit{word similarity}, which concerns whether a word can be substituted for another in most contexts. Empirical results have shown that, after jointly training $\phibf$ and $\psibf$ with SGNS, the dot products $\phibf^{\top}\phibf$ and $\psibf^{\top}\psibf$ capture word similarity well \cite{mikolov2013distributed,newell2019deconstructing}.

Later work has shown that this approximate parametrization of PMI holds for other popular word embedding algorithms \cite{kenyon2020deconstructing}, such as GloVe \cite{pennington2014glove} and FastText \cite{bojanowski2017enriching}. Moreover, there are works which have sought to learn embeddings by explicitly approximating PMI \cite{arora2016latent,newell2019deconstructing}.

\section{Identifying the Positive Pair}\label{section:identifying_positive_pair}

Thus far, we have posed the contrastive pre-training task as determining whether a given pair $(x,z) \in \Xcal \times \Xcal$ is a positive or negative sample. This approach is a staple of word embedding training, but, in computer vision, a modification of this idea is adopted. Instead of treating one pair at a time, a batch $\Bcal = \{(x_i,z_i)\}_{i=1}^B$ is sampled, where exactly one of the pairs is a positive sample and the remaining $B-1$ pairs are negative samples. The objective of the pre-training task then becomes to identify the positive pair.

This idea is captured by the triplet loss \cite{schroff2015facenet}, which was developed with the aim of learning features for images of faces. Given an input $x \in \Xcal$ (called the \textit{anchor}), another input $x^{(+)}$ is sampled such that $(x,x^{(+)})$ is a positive sample. Similarly, an input $x^{(-)}$ is sampled so that $(x,x^{(-)})$ is a negative sample.
\begin{definition}[Triplet Loss]
    Assume we are in the setting described above. Then, the \textbf{triplet loss} is
    \begin{equation}\label{eq:triplet_loss}
        \ell(\phibf; x,x^{(+)},x^{(-)}) = \max\big\{||\phibf(x) - \phibf(x^{(+)})||_2^2 - ||\phibf(x) - \phibf(x^{(-)})||_2^2 + \alpha, 0\big\},
    \end{equation}
    where $\alpha$ is a margin hyperparameter and $\phibf = \phibf_{\theta}: \Xcal \to \R^d$ is a neural network. 
\end{definition}
The loss is similar to (\ref{eq:hadsell_loss}) from \cite{hadsell2006dimensionality}. It encourages a short distance between the elements of a positive pair and a large distance between elements of a negative pair. The key distinction is that, in the triplet loss, positive and negative pairs are being compared against one another. As opposed to asking that elements of a negative pair have a distance at least $m$ from each other as in \cite{hadsell2006dimensionality}, we ask that the distance between them is large relative to the distance between elements of a positive pair.
 
The InfoNCE loss (also called the multi-class $N$-pair loss) \cite{sohn2016improved,oord2018representation} was proposed as an extension of the triplet loss. Rather than comparing a positive sample against a negative sample, we compare it against many negative samples. This brings us into the multiclass classification setting from Section \ref{section:classification}.

Let $x$ be an anchor example, $(x,x^{(+)})$ be a positive pair, and $\{(x,x_i^{(-)})\}_{i=1}^{B-1}$ be negative pairs. Suppose we have a model that produces a similarity function $s = s_{\theta}: \Xcal \times \Xcal \to \R$. Then, we can use the softmax function to compute our model's probability that $(x,x^{(+)})$ is indeed the true positive pair:
\begin{equation}\label{eq:mcnp_softmax}
    \frac{\exp(s(x,x^{(+)}))}{\exp(s(x,x^{(+)})) + \sum_{i=1}^{B-1} \exp(s(x,x_i^{(-)}))}.
\end{equation}
\begin{definition}[InfoNCE Loss]
    In the setting described above, the \textbf{InfoNCE loss} is the $B$-class cross-entropy loss
    \begin{equation}\label{eq:mcnp_loss}
    \ell_{\text{InfoNCE}}(s; x,x^{(+)}, \{x_i^{(-)}\}_{i=1}^{B-1}) = -\log \bigg(\frac{\exp(s(x,x^{(+)}))}{\exp(s(x,x^{(+)})) + \sum_{i=1}^{B-1} \exp(s(x,x_i^{(-)}))}\bigg).
\end{equation}
\end{definition}

This loss simultaneously encourages that positive pairs be assigned a high similarity and that negative pairs be assigned a low similarity. Since we are interested in learning representations of inputs, we take the similarity function to be the cosine similarity between feature vectors,
\begin{equation}\label{eq:mcnp_cossim}
    s(x,x') = \frac{\phibf(x)^{\top}\phibf(x')}{||\phibf(x)||\,||\phibf(x')||} \in [-1,1],
\end{equation}
where $\phibf = \phibf_{\theta}: \Xcal \to \R^d$ as usual.

Suppose that feature vectors are $\ell^2$-normalized. Then, the relationship between Euclidean distance and cosine similarity is clear:
\begin{equation}
    ||\phibf(x)-\phibf(x')||_2^2 = ||\phibf(x)||_2^2 + ||\phibf(x')||_2^2 - 2\phibf(x)^{\top} \phibf(x') = 2(1-\phibf(x)^{\top}\phibf(x')) \in [0,4].
\end{equation}


The authors of \cite{sohn2016improved} work in the supervised setting where a pair of inputs is positive if and only if the inputs belong to the same class. However, we are ultimately interested in feature learning on large unlabelled datasets, as was done for word embeddings. The authors of \cite{chen2020simple} adopt the InfoNCE loss along with a computationally efficient strategy for drawing positive and negative samples in the self-supervised setting. We detail this in the next section, where we also derive the optimum of the InfoNCE loss.

\section{Contrastive Learning with Data Augmentations}\label{section:contrastive_augmentation}

To perform contrastive learning with unlabelled data, it is necessary to define positive and negative pairs in such a way that they can be easily automatically extracted from the data. As we saw in Section \ref{section:word_embeddings}, this is done for word embeddings by defining positive pairs as genuine (context, target) pairs from the text corpus. Hence, we may identify positive pairs by iterating over the words in the text and extracting all pairs of words within a certain distance of each other. Negative samples are drawn from the (empirical) unigram distribution, which is specified by word counts in the corpus.

In computer vision, \cite{oord2018representation} and \cite{chen2020simple} propose making use of \textit{data augmentations} to define positive pairs. Examples of image data augmentations include Gaussian blur, colour jitter, random cropping and resizing, rotation, conversion to grayscale, and Sobel filtering\footnote{See Figure 4 from \cite{chen2020simple} for an illustration.}. Then, we can define two augmentations of the same source image to be a positive pair and two augmentations originating from different source images to be a negative pair. For example, two different random crops of the same dog are taken to be similar, while a random crop of a cat and a random crop of a dog are dissimilar.

With this simple definition of positive and negative pairs, we can immediately move to the setting of the InfoNCE loss from the previous section. Suppose we have an unlabelled pre-training image dataset $\Dcal_x$. Sample a batch of $B$ images $\{x_i\}_{i=1}^B$ from $\Dcal_x$. For each image $x_i$ in the batch, apply data augmentation to generate two transformed images\footnote{SimCLR employs a composition of Gaussian blur, random crop and resize, and colour jitter.} $\xtilde_{2i-1}, \xtilde_{2i}$. This process produces an augmented batch $\tilde{\Bcal} = \{\xtilde_j\}_{j=1}^{2B}$. For every $i \in [B]$, $(\xtilde_{2i-1},\xtilde_{2i})$ and $(\xtilde_{2i},\xtilde_{2i-1})$ are positive pairs. ($\xtilde_{2i-1}$ is the anchor in the first pair, while $\xtilde_{2i}$ is the anchor in the second pair). Meanwhile, $(\xtilde_{2i-1},\xtilde')$ and $(\xtilde_{2i},\xtilde')$ are negative samples for all $\xtilde' \in \tilde{\Bcal} \setminus \{\xtilde_{2i-1},\xtilde_{2i}\}$.
\begin{definition}[SimCLR Loss, \cite{chen2020simple}]
    In the above, setting the \textbf{SimCLR loss} is defined by
    \begin{equation}\label{eq:simclr_loss}
        \ell^{\tilde{\Bcal}}_{\mathrm{SimCLR}}(s) = \frac{1}{2B} \sum_{i=1}^B \ell_{\text{InfoNCE}}(s; \xtilde_{2i-1}, \xtilde_{2i}, \tilde{\Bcal} \setminus \{\xtilde_{2i-1},\xtilde_{2i}\}) + \ell_{\text{InfoNCE}}(s; \xtilde_{2i}, \xtilde_{2i-1}, \tilde{\Bcal} \setminus \{\xtilde_{2i-1},\xtilde_{2i}\}),
    \end{equation}
    where
    \begin{equation}
        s(x,z) = \frac{\phibf(x)^{\top}\phibf(z)}{\tau||\phibf(x)|| \, ||\phibf(z)||}
    \end{equation}
    with $\phibf = \phibf_{\theta}: \Xcal \to \R^d$ and $\tau > 0$ a hyperparameter.
\end{definition}
Recalling the intuition from InfoNCE loss, we can view each term in $\ell_{\mathrm{SimCLR}}$ as the categorical cross-entropy for a $(2B-1)$-class classification problem where the objective is to correctly identify the positive pair.

The authors of \cite{chen2020simple} pre-train the image encoder $\phibf$ with $\ell_{\mathrm{SimCLR}}$ on the ImageNet dataset \cite{russakovsky2015imagenet} using a batch size of 4096. Across a suite of image classification tasks, they find that linear probing with the resulting encoder has a higher accuracy compared to a fully supervised model\footnote{In fact, \cite{chen2020simple} conduct a modified version of linear probing. They pre-train $\phibf(\cdot) = g(\psibf(\cdot))$, where $g$ is a two-layer neural-network. Then, they perform linear probing on $\psibf$ rather than $\phibf$.}.

These impressive results invite a theoretical analysis of contrastive methods for computer vision. The authors of \cite{arora2019theoretical} prove upper bounds on the downstream task risk for models that are pre-trained with the InfoNCE objective. They assume that the data space $\Xcal$ can be subdivided into \textit{latent classes} such that positive pairs consist of two examples from the same latent class. Furthermore, they assume that the two examples in a positive pair are conditionally independent of each other given their latent class. The works \cite{tosh2021contrastive_a} and \cite{tosh2021contrastive_b} also prove similar guarantees under this conditional independence assumption. This assumption is problematic since in the case of SimCLR, a positive pair consists of two transformations of the same source image \cite{haochen2021provable}. 

The paper \cite{wang2020understanding} shows that, in the limit of infinite negative samples, the expected InfoNCE loss can be decomposed as the sum an \textit{alignment} term and a \textit{uniformity} term. The alignment term encourages the two elements of a positive pair to map to the same vector in feature space, while the uniformity term encourages that the pushforward measure of the data distribution by the embedding map $\phibf$ is the uniform distribution on the unit hypersphere in $\R^d$.

Under the assumption that data is generated by sampling from a low-dimensional space $\Zcal$ of latent variables and applying a (nonlinear) map $g: \Zcal \to \Xcal$, \cite{zimmermann2021contrastive} shows that contrastive learning with the InfoNCE loss can recover the latent variables up to orthogonal transformation. Recall from our discussion at the beginning of Chapter \ref{chapter:dim_reduction} that recovering a map from observables to latents is the main goal of dimensionality reduction. Furthermore, \cite{wen2021toward} shows that, when $g$ is a linear map plus additive noise, the resulting contrastive features are sensitive only to the latents and invariant to the noise.

In our study, we concentrate on works that derive the minima of contrastive loss functions in order to better understand the resulting representations. We proceed in similar fashion to Section \ref{section:word_embeddings} and follow the proofs in \cite{oord2018representation} and \cite{johnson2022contrastive}.

The data generating process from \cite{johnson2022contrastive} is as follows. Let $p(\cdot) = p_x(\cdot)$ denote the data distribution from which $\Dcal_x$ is sampled. Given an input $x \in \Xcal$, let $p(\cdot|x)$ denote the distribution over transformations of $x$, as specified by the data augmentations. Then, the probability of drawing $(\xtilde_1, \xtilde_2) \in \Xcal \times \Xcal$ as a positive pair is
\begin{equation}
    p_+(\xtilde_1,\xtilde_2) = \sum_{x \in \Xcal} p(x)p(\xtilde_1|x)p(\xtilde_2|x).
\end{equation}
To simplify the analysis, assume the negative samples $\xtilde_3,\dots,\xtilde_{2B}$ are drawn i.i.d.\@ from the marginal distribution $p$.\footnote{As is pointed out in \cite{johnson2022contrastive}, the actual data generating process of SimCLR takes $(\xtilde_3,\xtilde_4),\dots,(\xtilde_{2B-1},\xtilde_{2B})$ to be positive pairs.}

\begin{definition}[Positive-Pair Kernel, \cite{johnson2022contrastive}]
    Let $p,p_+$ be as above. Then, we define the \textbf{positive-pair kernel} $K: \Xcal \times \Xcal \to \R$ by
    \begin{equation}
        K_+(x,z) = \frac{p_+(x,z)}{p(x)p(z)} \hs\hs \forall \, x,z \in \Xcal.
    \end{equation}
\end{definition}

Assume $\Xcal$ is finite\footnote{Note that spaces of images are finite. For example, the set of RGB images of size $224 \times 224$ contains $(3 \cdot 224^2)^{256}$ elements.}. Then, the positive-pair kernel is PSD \cite{johnson2022contrastive}. Indeed, writing $\Xcal = \{x_1,\dots,x_{|\Xcal|}\}$, we can factor $K_+(x,z)$ as $\Phibf(x) \cdot \Phibf(z)$, where
\begin{equation}\label{eq:positive_pair_kernel_feature_map}
    \Phibf(x) = \bigg[\frac{p(x|x_1)\sqrt{p(x_1)}}{p(x)},\dots,\frac{p(x|x_{|\Xcal|})\sqrt{p(x_{|\Xcal|})}}{p(x)}\bigg].
\end{equation}

\begin{theorem}[SimCLR Implicitly Approximates a Kernel, \cite{johnson2022contrastive}]\label{thm:simclr_kernel}
    Consider the data generating process specified above. Then, the minimizer $\phibf$ of $\E[\ell_{\mathrm{SimCLR}}]$ satisfies
    \begin{equation}
        \exp\big(\phibf(x)^{\top} \phibf(z)/\tau\big) = C K_+(x,z) \hs\hs \forall \, x,z \in \Xcal
    \end{equation}
    for some constant $C > 0$.
\end{theorem}
That is, in the limit of infinite data (and under the simplifying assumption on how negative samples are generated), SimCLR learns a PSD kernel $\Khat(x,z) = \exp\big(\phibf(x)^{\top}\phibf(z)/\tau\big)$ (assuming w.l.o.g.\@ that $\phibf$ maps to the unit ball in $\R^d$) which approximates the positive-pair kernel $K_+(x,z)$ up to a constant factor. Just as we saw with the optimum of the SGNS risk in Section \ref{section:word_embeddings}, SimCLR learns representations so as to approximate the odds that a given pair is positive. 
\begin{proof}
    The expected loss is
    \begin{equation}
        \E[\ell_{\mathrm{SimCLR}}] = \underset{\substack{(\xtilde_1,\xtilde_2) \sim p_+(\cdot,\cdot) \\ \xtilde_3,\dots,\xtilde_{2B} \sim p(\cdot)}}{\E} \bigg[-\log \bigg(\frac{\exp\big(\phibf(\xtilde_1)^{\top}\phibf(\xtilde_2)/\tau\big)}{\sum_{k=2}^{2B} \exp\big(\phibf(\xtilde_1)^{\top}\phibf(\xtilde_k)/\tau\big)}\bigg)\bigg]
    \end{equation}
    For all $i \in \{2,\dots,2B\}$, let $p\big(y=\xtilde_i|\xtilde_1, (\xtilde_2,\dots,\xtilde_{2B})\big)$ denote the probability that $(\xtilde_1,\xtilde_i)$ is a positive sample given that exactly one of $\{(\xtilde_1,\xtilde_j)\}_{j=2}^{2B}$ is a positive sample. Then, the above expression is the expectation over all batches $\{\xtilde_j\}_{j=1}^{2B}$ of the cross-entropy between $p\big(y=\xtilde_2|\xtilde_1, (\xtilde_2,\dots,\xtilde_{2B})\big)$ and our model's estimate of that quantity. Therefore, at the minimum of the risk, we have, 
    \begin{equation}\label{eq:infonce_optim}
        \frac{\exp\big(\phibf(\xtilde_1)^{\top}\phibf(\xtilde_i)/\tau\big)}{\sum_{k=2}^{2B} \exp\big(\phibf(\xtilde_1)^{\top}\phibf(\xtilde_k)/\tau\big)} = p(y = \xtilde_i|\xtilde_1, (\xtilde_2,\dots,\xtilde_{2B})).
    \end{equation}
    Expanding the right-hand side, we have
    \begin{align*}
        p(y = \xtilde_i|\xtilde_1, (\xtilde_2,\dots,\xtilde_{2B})) &= \frac{p_+(\xtilde_1,\xtilde_i)\prod_{j=2}^{2B} \1_{j \neq i}\, p(\xtilde_j)}{\sum_{k=2}^{2B} p_+(\xtilde_1,\xtilde_k) \prod_{j=2}^{2B} \1_{j \neq k} \, p(\xtilde_j)} \\
        &= \frac{p_+(\xtilde_1,\xtilde_i)}{\sum_{k=2}^{2B} p_+(\xtilde_1,\xtilde_k) \frac{p(\xtilde_i)}{p(\xtilde_k)}} \\
        &= \frac{\frac{p_+(\xtilde_1,\xtilde_i)}{p(\xtilde_i)}}{\sum_{k=2}^{2B} \frac{p_+(\xtilde_1,\xtilde_k)}{p(\xtilde_k)}} \\
        &= \frac{\frac{p_+(\xtilde_1,\xtilde_i)}{p(\xtilde_1)p(\xtilde_i)}}{\sum_{k=2}^{2B} \frac{p_+(\xtilde_1,\xtilde_k)}{p(\xtilde_1)p(\xtilde_k)}}.
    \end{align*}
    From the above and (\ref{eq:infonce_optim}), we deduce
    \begin{equation}
        \label{eq:simclr_score}
        \exp\big(\phibf(x)^{\top}\phibf(z)/\tau\big) \propto \frac{p_+(x,z)}{p(x)p(z)} \hs\hs \forall \, x,z \in \Xcal.
    \end{equation}
\end{proof}
The authors of \cite{johnson2022contrastive} remark that the positive-pair kernel need not be explicitly specified on every pair of inputs as is the case for manifold learning and kernel approximation methods. Rather, the positive-pair kernel is implicitly specified by our definition of positive and negative pairs, which arise from data augmentation.

\begin{remark}
    Here, we worked with the PSD kernel $\Khat(x,z) = \exp(\phibf(x)^{\top}\phibf(z)/\tau)$, but the previous theorem can be generalized to any non-negative $K(x,z)$.
\end{remark}

\begin{remark}
    The same derivation as above can also be used to express the minimum of the empirical risk by replacing $p_+$ and $p$ with their count-based maximum likelihood estimates $\phat_+$ and $\phat$ respectively.
\end{remark}

\section{Connection to Dimensionality Reduction}

From Theorems \ref{thm:word2vec_kernel} and \ref{thm:simclr_kernel}, we see that SGNS and SimCLR are two prominent examples of contrastive methods which implicitly approximate a PSD kernel. This connection can be taken further to explicitly relate contrastive learning to the factorization of a Gram matrix (i.e., a kernel restricted to a finite space).

Assume that we are in the SimCLR setting with finite input space $\Xcal$. Define $N : = |\Xcal|$. Then, any kernel on $\Xcal$ is completely specified by its associated $N \times N$ Gram matrix. Similarly to ISOMAP and Laplacian Eigenmaps, the authors of \cite{haochen2021provable} adopt a graph formalism to encode similarity between elements of $\Xcal$. Let $G$ be a weighted graph (called the \textit{augmentation graph}) with vertices $\Xcal$ and edge weights $A_{xz} = p_+(x,z)$, i.e., the probability that $(x,z)$ is a positive pair. For every $x \in \Xcal$, its degree in the graph is
\begin{equation}
    \sum_{z \in \Xcal} A_{xz} = \sum_{z \in \Xcal} p_+(x,z) = \sum_{z \in \Xcal} \sum_{z' \in \Xcal} p(z')p(x|z')p(z|z') = p(x).
\end{equation}
Let $\Dbf \in \R^{N \times N}$ be the diagonal matrix with entries $D_{xx} = p(x)$. We then obtain the normalized adjacency matrix
\begin{equation}
    \bar \Abf := \Dbf^{-\frac{1}{2}} \Abf \Dbf^{-\frac{1}{2}}.
\end{equation}
That is, for $x,z \in \Xcal$,
\begin{equation}
    \bar{A}_{xz} = \frac{p_+(x,z)}{\sqrt{p(x)} \sqrt{p(z)}}.
\end{equation}
$\bar{\Abf}$ can be thought of as a normalized positive-pair kernel. Using a similar argument to (\ref{eq:positive_pair_kernel_feature_map}), it is easy to see that $\bar{\Abf}$ is PSD.

By the Eckart-Young-Mirsky Theorem \cite{eckart1936approximation},
\begin{equation}\label{eq:spectral_objective}
    \underset{\Fbf \in \R^{N \times d}}{\argmin} \, ||\bar{\Abf}-\Fbf \Fbf^{\top}||_F^2 = \Ubf_d \Lambdabf_d^{1/2}, 
\end{equation}
where $\Ubf_d$ is the matrix with the top $d$ normalized eigenvectors of $\bar{\Abf}$ as columns and $\Lambdabf_d = \mathrm{diag}(\lambda_1,\dots,\lambda_d)$ with $\lambda_1,\dots,\lambda_d$ the top $d$ eigenvalues of $\bar{\Abf}$. 

As was mentioned previously, spectral decomposition of the Gram matrix becomes prohibitively expensive for large $N$. At the beginning of this chapter, we posed contrastive learning as a remedy to expensive nonlinear dimensionality reduction techniques. The work \cite{haochen2021provable} puts this claim on rigourous footing by showing that minimizing a particular contrastive loss function is equivalent to solving (\ref{eq:spectral_objective}). 
\begin{definition}[Spectral Contrastive Loss]
    Let $\phibf: \Xcal \to \R^d$. Then, the \textbf{spectral contrastive loss} evaluated at $\phibf$ is
    \begin{equation}
        L_{\mathrm{spec}}(\phibf) := -2\underset{(x,x') \sim p_+(\cdot,\cdot)}{\E} \big[\phibf(x)^{\top}\phibf(x')\big] + \underset{z,z' \sim p(\cdot)}{\E}\big[\big(\phibf(z)^{\top}\phibf(z')\big)^2\big].
    \end{equation}
\end{definition}
\begin{theorem}[Equivalence between Spectral Contrastive Loss and Gram Matrix Factorization, \cite{haochen2021provable}]
    Let $\phibf: \Xcal \to \R^d$ and define the matrix $\Fbf \in \R^{N \times d}$ such that the $x$th row of $\Fbf$ is $\sqrt{p(x)}\phibf(x)$ for all $x \in \Xcal$. Then,
    \begin{equation}\label{eq:spectral_mf_equivalence}
        L_{\mathrm{spec}}(\phibf) = ||\bar{\Abf}-\Fbf \Fbf^{\top}||_F^2 + c
    \end{equation}
    for some constant $c \in \R$ which does not depend on $\phibf$.
\end{theorem}
\begin{proof}
    We have
    \begin{align*}
        ||\bar{\Abf} - \Fbf \Fbf^{\top}||_F^2 &= \sum_{x,z \in \Xcal} \bigg(\frac{p_+(x,z)}{\sqrt{p(x)}\sqrt{p(z)}} - \sqrt{p(x)}\sqrt{p(z)} \big(\phibf(x)^{\top}\phibf(z)\big)\bigg)^2 \\
        &= \sum_{x,z \in \Xcal} \bigg(\frac{p_+(x,z)^2}{p(x)p(z)} - 2p_+(x,z) \big(\phibf(x)^{\top}\phibf(z)\big) + p(x)p(z) \big(\phibf(x)^{\top}\phibf(z)\big)^2\bigg) \\
        &= -c - 2\underset{(x,x') \sim p_+(\cdot,\cdot)}{\E} \big[\phibf(x)^{\top}\phibf(x')\big] + \underset{z,z' \sim p(\cdot)}{\E}\big[\big(\phibf(z)^{\top}\phibf(z')\big)^2\big].
    \end{align*}
\end{proof}
From the first line of the proof, it is clear that for $\phibf^*$ the minimizer of $L_{\mathrm{spec}}$, we have the approximation
\begin{equation}
    \phibf^*(x)^{\top} \phibf^*(z) \approx K_+(x,z).
\end{equation}

The benefit of working with the spectral contrastive loss rather than an explicit spectral decomposition or the Nyström method is twofold. First, the spectral contrastive loss does not require the explicit computation of the positive-pair kernel (which is intractable) on all pairs of inputs. Instead, it suffices to sample positive and negative pairs. Second, we avoid the computational cost associated with a spectral decomposition or the evaluation of Nyström features.

Through their graph formalism, the authors of \cite{haochen2021provable} prove learning bounds for pre-training with spectral contrastive loss and linear probing under much weaker assumptions than prior work. Namely, the unrealistic conditional independence assumption from \cite{arora2019theoretical,tosh2021contrastive_a,tosh2021contrastive_b} is not necessary. Instead, the key assumption is that the vertices can be partitioned into $m$ clusters which have little overlap (i.e., edges between different clusters have small weight).

\begin{definition}[Dirichlet Conductance]
    Given an augmentation graph $G = (\Xcal, \Abf)$ and a subset of vertices $S \subseteq \Xcal$, the \textbf{Dirichlet conductance} of $S$ is
    \begin{equation}
        \nu_G(S) = \frac{\sum_{x \in S} \sum_{x' \notin S} A_{xx'}}{\sum_{x \in S} p(x)}.
    \end{equation}
\end{definition}
Intuitively, the Dirichlet conductance of $S$ can be thought of as measuring the amount of "overlap" between $S$ and $\Xcal \setminus S$. 

\begin{definition}[Sparsest $i$-partition]
    Let $G = (\Xcal,\Abf)$ be an augmentation graph. For every $i \in \{2,\dots,N\}$, define 
    \begin{equation}
        \rho_i := \min_{S_1,\dots,S_i} \max\big\{\nu_G(S_1),\dots,\nu_G(S_i)\big\},
    \end{equation}
    where $S_1,\dots,S_i$ are nonempty sets which form a partition of $\Xcal$.
\end{definition}
We see that $\rho_i$ measures the amount of overlap when $\Xcal$ is partitioned into $i$ clusters. This allows for the formalization of the key assumption that there are at most $m$ clusters in the data.

\begin{assumption}\label{asmp:m_clusters}
    Fix a constant $m \in \N$ and $\rho \in [0,1)$. Assume that $\rho_{m+1} \geq \rho$. 
\end{assumption}

\begin{assumption}\label{asmp:invariance_to_augs}
    Let $x \in \Xcal$ and let $h^*: \Xcal \to [C]$ be the labelling function for a downstream task of interest. Fix $\alpha \in [0,1)$. Assume there exists a classifier $h: \Xcal \to [C]$ such that
    \begin{equation}
        \underset{\xtilde \sim p(\cdot|x)}{\E}\big[\1\big(h(x) \neq h^*(\xtilde)\big)\big] \leq \alpha.
    \end{equation}
\end{assumption}
In other words, we assume that data augmentation changes the label of an example with probability at most $\alpha$. Recall that in the supervised setting, one can simply form positive pairs from examples belonging to the same class. This assumption states that data augmentation is a reasonable proxy in the absence of labelled data.

\begin{assumption}\label{asmp:realizability}
    Let $\Fcal$ be a hypothesis class of functions from $\Xcal$ to $\R^d$. Assume that $\Fcal$ contains a global minimizer of $L_{\mathrm{spec}}$.
\end{assumption}

\begin{definition}[Linear Probing Error]
    Let $\phibf: \Xcal \to \R^d$ be an encoder and $h^*: \Xcal \to [C]$ be the labelling function for a downstream task. For every $\Wbf \in \R^{C \times d}$, define the downstream linear classifier
    \begin{equation}
        h_{\phibf,\Wbf}(x) := \underset{j \in [C]}{\argmax} \big(\Wbf \phibf(x)\big)_j.
    \end{equation}
    Then, the \textbf{linear probing error} of $\phibf$ is
    \begin{equation}
        \Ecal(\phibf) = \min_{\Wbf \in \R^{C \times d}} \underset{x \sim p(\cdot)}{\E} \big[\1\big(h_{\phibf,\Wbf}(x) \neq h^*(x)\big)\big].
    \end{equation}
\end{definition}

\begin{theorem}[Linear Probing Error Under Infinite Data, \cite{haochen2021provable}]
    Suppose $d \geq 2\max\{C,m\}$ and Assumptions \ref{asmp:m_clusters} and \ref{asmp:invariance_to_augs} hold. Let $\Fcal$ be as in Assumption \ref{asmp:realizability} with $\phibf_{pop}^* \in \Fcal$ the global minimizer of $L_{\mathrm{spec}}$. Then,
    \begin{equation}
        \Ecal(\phibf_{pop}^*) \leq \tilde{O}\bigg(\frac{\alpha}{\rho^2}\bigg).
    \end{equation}
\end{theorem}
In other words, assuming infinite pre-training and downstream data, the downstream error is controlled (up to logarithmic factors) by the invariance of the labelling function to data augmentations and the number of clusters in the data.

In practice, we have a finite pre-training dataset $\Dcal_{pre} = \{x_1,\dots,x_{N_{pre}}\}$ and downstream task dataset $\Dcal_{down} = \{(x_1',y_1),\dots,(x_{N_{down}}',y_{N_{down}})\}$ with $N_{down} \ll N_{pre}$. With this in mind, the authors of \cite{haochen2021provable} also prove finite-sample learning bounds using tools from statistical learning theory. We defer to their paper for further details.

Relatedly, the authors of \cite{balestriero2022contrastive} work with a different choice of graph\footnote{Unlike \cite{haochen2021provable}, edge weights are not a modification of the positive-pair kernel.} $G'$ to explicitly characterize optimal representations that result from minimizing a modification of the SimCLR loss. They show a correspondence between SimCLR and ISOMAP on $G'$. Furthermore, they assess the conditions under which self-supervised representations will aid downstream task performance.

\section{Learning Mercer Eigenfunctions}\label{section:learning_mercer_eigenfunctions}

In the previous section, we related contrastive features to the eigenvectors of the Gram matrix associated with $K_+$. Recall from Example \ref{ex:mercer_finite} that these eigenvectors are in fact the Mercer eigenfunctions associated with $K_+$ with respect to the counting measure. In Section \ref{section:kernel_approx}, we saw that the Nyström method is a means of approximating Mercer eigenfunctions with respect to the data distribution $p$. The paper \cite{johnson2022contrastive} shows that the latter eigenfunctions in fact comprise the the optimal low-rank encoder for a linear probe when the labelling function $h^*$ is approximately invariant to data augmentations.
\begin{theorem}[Optimality of Mercer Eigenfunctions, \cite{johnson2022contrastive}]
    Let $\Hcal_{\phibf} := \{x \mapsto \wbf^{\top} \phibf(x): \wbf \in \R^d\}$ be the set of linear predictors composed with an encoder $\phibf$. Let
    \begin{equation*}
        K_+(x,z) = \sum_{j=1}^N \lambda_j \psi_j(x) \psi_j(z),
    \end{equation*}
    where $\lambda_1 \geq \dots \geq \lambda_N$, be the Mercer decomposition of $K_+$ with respect to the measure $p$ as in Theorem \ref{thm:mercer}. Define $\phibf(\cdot) = [\psi_1(\cdot),\dots,\psi_d(\cdot)]$. Then, $\Hcal_{\phibf}$ is the set of maximally invariant predictors in the following sense:
    \begin{equation}
        \Hcal_{\phibf} = \underset{\dim(\Hcal) = d}{\argmin} \max_{\substack{h \in \Hcal \\ \E_p[h(x)^2] = 1}} \underset{(x_1,x_2) \sim p_+}{\E} \bigg[\big(h(x_1) - h(x_2)\big)^2\bigg].
    \end{equation}
    Furthermore, let $\eps > 0$ and
    \begin{equation}
        S_{\eps} := \bigg\{h: \Xcal \to \R: \underset{(x_1,x_2) \sim p_+}{\E} \bigg[\big(h(x_1)-h(x_2)\big)^2\bigg] \leq \eps \bigg\}.
    \end{equation}
    Then, $\Hcal_{\phibf}$ is the optimal hypothesis class of dimension $d$ for approximating invariant labelling functions, i.e.,
    \begin{equation}\label{eq:mercer_least_squares_optimality}
        \Hcal_{\phibf} = \underset{\dim(\Hcal) = d}{\argmin} \, \max_{h^* \in S_{\eps}} \,\min_{h \in \Fcal}\, \E_p \bigg[\big(h(x) - h^*(x)\big)^2\bigg].
    \end{equation}
\end{theorem}
In addition to factoring their associated kernel, Mercer eigenfunctions also serve as optimal features for downstream tasks. In fact, any basis of $\mathrm{span}\{\psi_1,\dots,\psi_d\}$ will be optimal.

The work \cite{deng2022neuralef} presents an algorithm that recovers the top $d$ eigenfunctions of $K_+$ with respect to $p$. Unlike the Nyström method, it does not require evaluating the intractable $K_+$. The method, dubbed \textit{Neural Eigenfunctions} (NeuralEF), numerically solves $d$ optimization problems (one for each eigenfunction) using neural network encoders as function approximators.

\begin{theorem}[Optimization Problem for Mercer Eigenfunctions, \cite{deng2022neuralef}]\label{thm:neuralef_optimization}
    Let $K: \Xcal \times \Xcal \to \R$ be a PSD kernel. Let $\{(\lambda_j,\psi_j)\}_{j=1}^N$ denote the eigenpairs of $T_K$ with respect to $p$ (in order of decreasing eigenvalues). For all $i,j \in [d]$, let
    \begin{equation}
        R_{ij} := \int_{\Xcal} \int_{\Xcal} \hat{\psi}_i(x) K(x,z) \hat{\psi}_j(z) p(z) p(x) \, dz \, dx
    \end{equation}
    and
    \begin{equation}
        C_j := \int_{\Xcal} \hat{\psi}_j(x) \hat{\psi}_j(x)p(x) \,dx.
    \end{equation}
    Then, $(\psi_1,\dots,\psi_d)$ solves the simultaneous maximization problems
    \begin{equation}\label{eq:simultaneous_maximization}
        \max_{\hat{\psi}_j} R_{jj}, \hs\hs j \in \{1,\dots,d\}
    \end{equation}
    subject to the constraints $C_j = 1$ and $R_{ij} = 0$ for all $i \in \{1,\dots,j-1\}$.
\end{theorem}
\begin{proof}
    We have the Mercer decomposition
    \begin{equation}
        K(x,z) = \sum_{j=1}^N \lambda_j \psi_j(x) \psi_j(z) \hs\hs \forall \, x,z \in \Xcal.
    \end{equation}
    We may solve the maximization problems (\ref{eq:simultaneous_maximization}) sequentially from $j=1$ to $j=d$. We have
    \begin{align*}
        R_{11} &= \int_{\Xcal} \int_{\Xcal} \hat{\psi_1}(x) \bigg(\sum_{j=1}^N \lambda_j \psi_j(x) \psi_j(z)\bigg) \hat{\psi}_1(z)p(z)p(x) \, dz \, dx \\
        &= \sum_{j=1}^N \lambda_j \int_{\Xcal} \int_{\Xcal} \hat{\psi}_1(x) \psi_j(x) \psi_j(z) \hat{\psi}_1(z) p(z) p(x) \, dz \, dx \\
        &= \sum_{j=1}^N \lambda_j \bigg(\int_{\Xcal} \hat{\psi}_1(x) \psi_j(x) p(x) \, dx\bigg)\bigg(\int_{\Xcal} \hat{\psi}_1(z) \psi_j(z)p(z) \, dz\bigg) \\
        &= \sum_{j=1}^N \lambda_j \ev{\hat{\psi}_1, \psi_j}^2.
    \end{align*}
    We can write $\hat{\psi}_1$ in the orthonormal basis $\{\psi_i\}_{i=1}^N$ as
    \begin{equation}
        \hat{\psi}_1 = \sum_{i=1}^N a_i \psi_i.
    \end{equation}
    Hence,
    \begin{equation}
        R_{11} = \sum_{j=1}^N \lambda_j \bigg \langle \sum_{i=1}^N a_i \psi_i, \psi_j \bigg \rangle^2 = \sum_{j=1}^N \lambda_j a_j^2.
    \end{equation}
    Since $C_1 = 1$,
    \begin{equation}
        \ev{\hat{\psi}_1, \hat{\psi}_1} = \bigg \langle \sum_{i=1}^N a_i \psi_i, \sum_{j=1}^N a_j \psi_j\bigg \rangle = \sum_{j=1}^N a_j^2 = 1.
    \end{equation}
    To maximize $R_{11}$, we set $a_1 = 1$ and $a_k = 0$ for all $k \in \{2,\dots,N\}$. This gives $\hat{\psi}_1 = \psi_1$.

    For the second problem, we must find $\hat{\psi}_2$ such that $R_{12} = 0$. That is,
    \begin{align*}
        &\int_{\Xcal} \int_{\Xcal} \hat{\psi}_1(x)K(x,z)\hat{\psi}_2(z)p(z)p(x) \, dz \, dx = 0 \\
        &\iff \int_{\Xcal} \hat{\psi}_2(z)p(z) \int_{\Xcal} \hat{\psi_1}(x) K(x,z)p(x) \, dx \, dz = 0 \\
        &\iff \int_{\Xcal} \hat{\psi}_2(z)p(z) \int_{\Xcal} \psi_1(x) K(x,z)p(x) \, dx \, dz = 0 \\
        &\iff \int_{\Xcal} \hat{\psi}_2(z)p(z) \lambda_1 \psi_1(z) \, dz = 0 \\
        &\iff \ev{\psi_1, \hat{\psi}_2} = 0.
    \end{align*}
    From here, we apply the same approach as the first problem. Writing $\hat{\psi}_2$ in the orthonormal basis $\{\psi_i\}_{i=1}^N$ as 
    \begin{equation}
        \hat{\psi}_2 = \sum_{i=1}^N b_i \psi_i,
    \end{equation}
    we see that the orthogonality constraint forces $b_1 = 0$. Taking into account $C_2 = 1$, the optimal solution is to take $b_2 = 1$ and $b_k = 0$ for $k \geq 3$. We proceed similarly for the remaining optimization problems from $j=3$ to $j=N$.
\end{proof}

To approximately solve the optimization problems outlined above, the authors use Monte Carlo sampling to approximate the integrals and neural networks as approximators for the desired functions. Let $\phi_{\theta_1},\dots,\phi_{\theta_d}$ be neural network encoders which map from $\Xcal$ to $\R$. These networks are taken to have the same architecture but distinct parameters. A batch $\Bcal = \{x_b\}_{b=1}^B$ is sampled, from which $R_{ij}$ is approximated by
\begin{equation}
    \tilde{R}_{ij} = \frac{1}{B^2} \sum_{b=1}^B \sum_{b'=1}^B \phi_{\theta_i}(x_b)K(x_b,x_{b'})g_{\theta_j}(x_{b'}) =: \frac{1}{B^2} \phibf_{\theta_i}^{\Bcal} \Gbf^{\Bcal} \phibf_{\theta_j}^{\Bcal}.
\end{equation}
for all $i,j \in [d]$, where $\phibf_{\theta_i}^{\Bcal} = [\phi_{\theta_i}(x_1),\dots,\phi_{\theta_i}(x_B)]^{\top}$ and $\Gbf^{\Bcal}$ is the Gram matrix of $K$ associated with the batch $\Bcal$.

We approximate the constraint $C_j = 1$ for all $j \in [d]$ by
\begin{equation}
    \frac{1}{B} \sum_{b=1}^B \phi_{\theta_j}(x_b) \phi_{\theta_j}(x_b) = 1.
\end{equation}
To ensure this holds, we normalize the output of each neural network via
\begin{equation}
    \phi_{\theta_j} \mapsto \frac{\phi_{\theta_j}}{\sqrt{\frac{1}{B} \sum_{b=1}^B \phi_{\theta_j}(x_b)}}.
\end{equation}

Then, the loss function applied to the batch seeks to maximize the terms $\tilde{R}_{jj}$ for all $j \in [d]$ while enforcing $\tilde{R}_{ij} = 0$ for $i \neq j$. We have
\begin{equation}
    \ell(\phi_1,\dots,\phi_d) = -\frac{1}{B^2} \bigg(\sum_{j=1}^d \phibf_{\theta_j}^{\Bcal} \Gbf^{\Bcal} \phibf_{\theta_j}^{\Bcal} - \sum_{i=1}^{j-1} \frac{\big(\mathrm{sg}(\phibf_{\theta_i}^{\Bcal})^{\top} \Gbf^{\Bcal} \phibf_{\theta_j}^{\Bcal}\big)^2}{\mathrm{sg}\big((\phibf_{\theta_i}^{\Bcal})^{\top} \Gbf^{\Bcal} \phibf_{\theta_i}^{\Bcal}\big)}\bigg),
\end{equation}
where $\mathrm{sg}$ denotes the stop-gradient operation, i.e., the argument of $\mathrm{sg}$ is treated as a constant during the calculation of the gradient for SGD. This operation is reflective of the iterative way in which we solved the optimization problems in the proof of Theorem \ref{thm:neuralef_optimization}.

The follow-up work \cite{deng2022neural} introduces Neural Eigenmap, an algorithm which trains a single neural network $\phibf: \Xcal \to \R^d$ to solve a related optimization problem which has $[\psi_1,\dots,\psi_d]$ as the optimal solution. The authors use Neural Eigenmap (with the positive-pair kernel) to train an image encoder on the ImageNet dataset \cite{russakovsky2015imagenet}, as is common practice in the literature. They find linear probing accuracy to be competitive with contrastive methods such as SimCLR. Moreover, on a collection of image retrieval tasks, the authors find that the dimension of the representations produced by the Neural Eigenmap encoder can be truncated by up to 16$\times$ that of other baselines, such as the spectral contrastive loss. They point to this as evidence of the benefit of approximating the eigenfunctions of the positive-pair kernel in order.

%% file: Chapters/conclusion.tex
\chapter{Conclusion}

In this thesis, we addressed recent progress in the development of a theoretical foundation for contrastive learning. We emphasized the finding that popular contrastive learning methods in computer vision and natural language processing --- namely word2vec and SimCLR --- learn to approximate a positive semidefinite kernel. This kernel is implicitly defined via the process by which positive and negative pairs are sampled for the contrastive pretext task. Specifically, it encodes the odds that a given pair is a positive sample. This relationship between contrastive learning and kernel approximation arises naturally when the contrastive pretext task is analyzed as a supervised classification problem that is solved with logistic regression. 

Taking this further, we detailed the spectral contrastive loss function, which explicitly connects contrastive learning to the factorization of a Gram matrix, the latter of which is a common theme in unsupervised dimensionality reduction. This theoretical result confirms that contrastive learning achieves the same objective as unsupervised dimensionality reduction while overcoming the latter's key shortcomings. Contrastive learning operates with implicitly defined kernels that incorporate additional user-specified knowledge (e.g., label-preserving data augmentations), can generalize to unseen data with neural network models, and avoids the computational expense of a spectral decomposition.

Finally, we addressed how the above connections have inspired work on explicitly learning features from kernels. Under the assumption that data augmentations are approximately label-preserving, top Mercer eigenfunctions of the positive-pair kernel comprise optimal low-dimensional features for linear probing. Recent approaches such as Neural Eigenfunctions and Neural Eigenmap pose the approximation of Mercer eigenfunctions as an optimization problem to be solved with neural networks.

Our intention is for this thesis to serve as a reference for researchers entering this area. The theory of contrastive learning is by no means complete and there remain several directions of investigation. For example, a comprehensive justification of hyperparameter choices made in contrastive methods such as SimCLR is lacking. The theoretical work \cite{arora2019theoretical} suggests that larger batch sizes hinder downstream performance, but, as noted by \cite{wang2020understanding,zimmermann2021contrastive}, empirical work has found the opposite to be true \cite{wu2018unsupervised,he2020momentum,chen2020simple}. Additionally, SimCLR discards the final two layers of its neural network encoder during linear probing. To our knowledge, theoretical work does not take this into account and instead focuses on linear probing on top of the encoder's final layer. By the same token, there are few theoretical investigations of fine-tuning \cite{kumar2022fine}.

It is also important to situate contrastive learning in the broader contexts of self-supervised learning and deep learning. In parallel to the developments discussed in this thesis, several non-contrastive self-supervised representation learning methods have been proposed \cite{zbontar2021barlow,grill2020bootstrap,bardes2022vicreg,caron2021emerging}. From a theoretical perspective, there is interest in comparing these to contrastive methods and determining which approach should be preferred given pre-training data and a set of downstream tasks \cite{balestriero2022contrastive}. Ultimately, all of these approaches are reliant on neural networks, whose generalization ability is not yet completely understood \cite{belkin2021fit}. Hence, we expect that progress made on the question of generalization will also serve to enhance the theory of contrastive learning.